\documentclass{article}

\PassOptionsToPackage{numbers, compress}{natbib}
\usepackage[main, final]{neurips_2026}

\usepackage[utf8]{inputenc} 
\usepackage[T1]{fontenc}    
\usepackage{hyperref}       
\usepackage{url}            
\usepackage{booktabs}       
\usepackage{amsfonts}       
\usepackage{nicefrac}       
\usepackage{microtype}      
\usepackage{xcolor}         
\usepackage{colortbl}      
\usepackage{kotex}
\usepackage{wrapfig}
\usepackage{enumitem}
\usepackage{multirow} 
\usepackage{xspace}
\usepackage{bbm} 
\usepackage{booktabs}
\usepackage{multirow}
\usepackage[table]{xcolor}
\usepackage{graphicx} 
\usepackage{amssymb} 
\usepackage{tabularx}
\usepackage{wrapfig}
\usepackage[toc,page]{appendix}
\usepackage{etoc}
\usepackage{algorithm}
\usepackage{bm}
\usepackage{algpseudocode}
\newcolumntype{C}{>{\centering\arraybackslash}X}

\usepackage{graphicx}
\usepackage{booktabs} 

\usepackage{hyperref}

\newcommand{\proposed}{\texttt{SafeCut}\xspace}

\usepackage{amsmath, amssymb, amsthm}
\usepackage{graphicx}                       
\usepackage{wrapfig}
\usepackage[most]{tcolorbox}                
\usepackage[framemethod=TikZ]{mdframed}     
\usepackage{caption}
\usepackage{subcaption}
\usepackage{pifont}
\usepackage{enumitem}
\newtcolorbox{contributionbox}{
    enhanced,
    colback=cyan!3,
    colframe=cyan!3,
    borderline west={1.2pt}{0pt}{cyan!60!black},
    sharp corners,
    boxrule=0pt,
    left=5pt,
    right=4pt,
    top=4pt,
    bottom=4pt
}

\newtcolorbox{namedbox}[2][]{enhanced,breakable,
  colback=cyan!4, colframe=cyan!65!black,
  boxed title style={colback=cyan!65!black, colframe=cyan!65!black},
  coltitle=white, fonttitle=\small,
  title={#2},
  attach boxed title to top left={xshift=1.0mm, yshift*=-\tcboxedtitleheight/2},
  boxsep=0.1pt, left=2pt, right=2pt, top=4pt, bottom=4pt, arc=3pt,
  #1}

\mdfdefinestyle{mathblockframe}{
  linecolor=gray!40,
  linewidth=0.5pt,
  backgroundcolor=gray!8,
  roundcorner=3pt,
  innertopmargin=5pt,
  innerbottommargin=5pt,
  innerleftmargin=3pt,
  innerrightmargin=3pt,
  skipabove=3pt,
  skipbelow=3pt,
}

\newmdtheoremenv[style=mathblockframe]{theorem}{Theorem}[section]
\newmdtheoremenv[style=mathblockframe]{proposition}{Proposition}[section]
\newmdtheoremenv[style=mathblockframe]{lemma}{Lemma}[section]
\newmdtheoremenv[style=mathblockframe]{corollary}{Corollary}[section]
\newmdtheoremenv[style=mathblockframe]{property}{Property}[section]

\theoremstyle{definition}
\newmdtheoremenv[style=mathblockframe]{definition}{Definition}[section]
\newmdtheoremenv[style=mathblockframe]{assumption}{Assumption}[section]
\newmdtheoremenv[style=mathblockframe]{observation}{Observation}[section]

\theoremstyle{remark}
\newmdtheoremenv[style=mathblockframe]{remark}{Remark}[section]

\usepackage[textsize=tiny]{todonotes}
\usepackage{float}
\usepackage{stfloats}

\title{Safeguarding Mutual Correction in \\ Source-Free Domain Adaptation via Cut Statistics}

\author{%
  Seongjun Lee \\
  Department of Artificial Intelligence \\
  Korea University \\
  Seoul, Republic of Korea \\
  \texttt{pyoung7307@korea.ac.kr} \\
  \And
  Changhee Lee\thanks{Corresponding author.} \\
  Department of Artificial Intelligence \\
  Korea University \\
  Seoul, Republic of Korea \\
  \texttt{changheelee@korea.ac.kr} \\
}

\begin{document}

\maketitle
\begin{abstract}
    Source-Free Domain Adaptation (SFDA) aims to adapt a source-pretrained model to an unlabeled target domain without access to the original source domain. While early single-model approaches rely on self-refinement, they are inherently susceptible to confirmation bias and struggle to correct their own systematic errors. To overcome this limitation, recent methods introduce Vision-Language (ViL) models as external knowledge sources. However, these approaches operate in a largely unidirectional paradigm, using the ViL model primarily to supervise the source-pretrained model. This overlooks a key structural property: the two models exhibit distinct failure modes -- where one produces an incorrect prediction, the other may produce a correct one, creating a natural opportunity for mutual correction within the target domain. Yet, without ground-truth labels, identifying which model is correct on any given sample is non-trivial, and naively exchanging predictions risks propagating errors across models. To address this challenge, we propose \proposed, a novel approach that leverages the cut statistic as a label-free measure of prediction reliability to gate cross-model supervision. Our approach dynamically controls both the direction and strength of supervision based on relative reliability, selectively amplifying true corrections while suppressing miscorrections on a per-sample basis. We further provide theoretical justification showing that this reliability-gated mechanism guarantees a net-positive correction signal. Extensive experiments across diverse SFDA benchmarks demonstrate that \proposed achieves state-of-the-art performance, highlighting the effectiveness of safeguarding mutual correction in SFDA via cut statistics. Codes are available \textcolor{blue}{\underline{\href{https://github.com/SeongJunLee-net/SafeCut_NIPS26}{here}}}.
\end{abstract}

\section{Introduction}
\label{sec:intro}

One of the central challenges in deploying deep learning models is the presence of distribution shift between training and deployment environments.
Unsupervised Domain Adaptation (UDA) addresses this issue by aligning source and target distributions without relying on target labels~\citep{dann_ganin_icml,mkmdd_long_icml}. However, it requires access to source data -- an assumption often violated in privacy- and regulation-sensitive domains such as healthcare and finance. Source-Free Domain Adaptation (SFDA) has therefore emerged as a practical alternative, operating under the \textit{stricter constraint} that only a source-pretrained model and unlabeled target data are accessible, while both source data and target labels are \textit{completely inaccessible}~\citep{shot_liang_icml}.

SFDA assumes that, despite distribution shifts, the source and target domains share an underlying semantic structure, allowing knowledge learned from the source domain to remain partially transferable. Early approaches capitalize on this shared structure by leveraging the structural proximity of predictions or latent features induced by the source-pretrained model to derive self-supervised adaptation signals on the target domain~\citep{shot_liang_icml,nrc_yang_nips,aad_yang_nips,adacontrast_chen_cvpr,plue_litrico_cvpr,cowa_lee_icml}. However, because these single-model approaches rely exclusively on self-refinement, they are highly susceptible to confirmation bias and inherently struggle to correct their own systematic errors, as incorrect predictions are often distributed throughout the target domain, making them difficult to correct through local proximity or consistency-based supervision~\citep{cowa_lee_icml, elr_yi_iclr}.

\begin{wrapfigure}{r}{0.49\textwidth}
    \vspace{-2mm}
    \centering
    \includegraphics[width=0.49\textwidth]{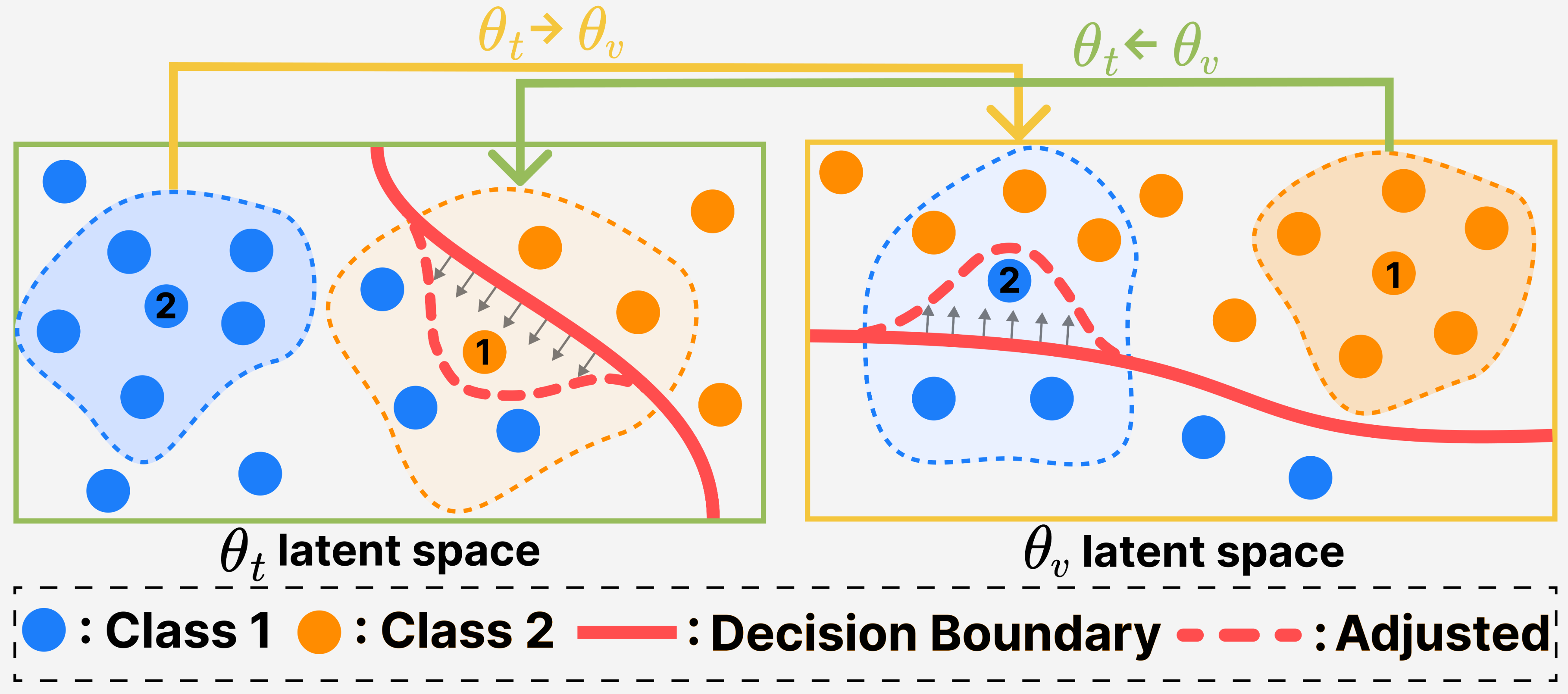}
    \caption{\textbf{Reliability-based Mutual Correction.}
    The model with lower cut corrects the other (Sample 1: $\theta_v\!\to\!\theta_t$, Sample 2: $\theta_t\!\to\!\theta_v$).}
    \label{fig:motivation} \vspace{-2mm}
\end{wrapfigure}
Recent work introduces Vision-Language (ViL) models as an external knowledge source, enabling cross-model supervision beyond self-training~\cite{difo_tang_cvpr,prode_tang_iclr}. Their rich, generalized representations provide corrective signals in regions where the source-pretrained model is uncertain or biased. Despite these advantages, existing approaches treat ViL as a universally more reliable knowledge source. In practice, however, ViL models may have distinct failure modes from the source-pretrained models -- cases where one model errs while the other is correct -- and their relative reliability varies across samples and target domains. Unconditionally accepting ViL predictions ignores this structural variability, which limits potential gains from adaptation. Moreover, such one-sided reliance can hinder generalization in domains where ViL models are inherently less reliable, such as medical imaging~\citep{reclip_hu_wacv,cpclip_yu_miccai}. More fundamentally, it overlooks the opportunity to leverage these distinct failure modes for mutual correction.

To address this, we propose a novel approach, \textbf{Safe}guarding mutual correction via \textbf{Cut} statistics (\proposed), which establishes a positive feedback loop through \textit{reliability-gated knowledge transfer} on the target domain. Using the \textit{cut statistic}~\citep{cutstats_muhlenbach_jiis} as a label-free measure of local prediction consistency, our method dynamically controls both the direction and strength of cross-model supervision (Figure~\ref{fig:motivation}). This selectively amplifies corrective signals while suppressing miscorrections, enabling the two models to refine each other in a principled manner. As a result, \proposed achieves state-of-the-art performance across diverse domain adaptation benchmarks.
\vspace{-1mm}
\begin{contributionbox}
\textbf{Our contributions are summarized as follows:}
\begin{itemize}[leftmargin=1.5em]
    \item We provide extensive empirical evidence that the cut statistic is strongly and negatively correlated with prediction accuracy in domain adaptation settings, and further establish a theoretical foundation based on cut margins to exploit different failure modes.
    \item Building on this insight, we propose \proposed, a reliability-gated method that leverages the cut statistic to dynamically control both the direction and strength of cross-model supervision.
    \item We demonstrate that \proposed achieves state-of-the-art performance across diverse domain adaptation benchmarks, highlighting the effectiveness of exploiting complementary failure modes for SFDA.
\end{itemize}
\end{contributionbox}

\section{Related Works}\label{sec:rel}
\paragraph{Source-Free Domain Adaptation.} Early work aligns target representations via self-supervised pseudo-labeling~\cite{shot_liang_icml}, and subsequent methods exploit target neighborhood structure for prediction consistency, pseudo-label refinement, and contrastive adaptation~\cite{nrc_yang_nips,aad_yang_nips,adacontrast_chen_cvpr}. Alternatively, several approaches estimate prediction reliability via GMM-based confidence or neighborhood uncertainty, and subsequently use these quantities for self-supervision~\cite{plue_litrico_cvpr,cowa_lee_icml}. However, by relying solely on the source model's predictions and the resulting latent structures, these methods are inherently constrained by the model's own blind spots and confirmation bias. To address this, recent work introduces ViL as an auxiliary knowledge source, focusing on adapting ViL to the target domain or denoising its pseudo-labels~\cite{difo_tang_cvpr,prode_tang_iclr}. These methods are specifically tailored to the zero-shot capabilities of ViL and do not explore their extension to more general counterpart models; furthermore, they leave the critical question of how effective knowledge transfer can occur between the two models largely unaddressed.

\paragraph{Cross-Model Supervision.} Cross-model supervision has been explored in mutual learning~\cite{DeepMutualLearning_Zhang_cvpr} and co-teaching~\cite{coteaching_han_nips} frameworks, where two models iteratively refine each other -- either by exchanging predictions or by selecting reliable samples for the peer to learn from. However, these approaches fundamentally rely on labeled data or noisy label supervision, making them inapplicable to the fully unsupervised SFDA setting. More recently, weak-to-strong generalization~\cite{w2s_burns_openai, weak2strong_theory_lang_nips} demonstrates that a weaker model's imperfect supervision can nonetheless elicit strong generalization from a more capable peer. \proposed builds on this insight and extends it to the label-free SFDA setting, leveraging the cut statistic to dynamically route supervision based on local prediction consistency — enabling reliable cross-model correction without any label information.


\section{Mutual Correction in Source-free Domain Adaptation}
\label{sec:prelim}

\subsection{Problem Setting}
\label{sec:setting}

We consider a $C$-class classification task over the input space $\mathcal{X}$ and the label space $\mathcal{Y} = \{1, \dots, C\}$. In the \textit{source-free domain adaptation} (SFDA) setting, we are provided with a source model $\theta_s: \mathcal{X} \to \mathbb{R}^C$ pre-trained on a labeled source domain $\mathcal{D}_s$. Our objective is to adapt this model to an unlabeled target domain $\mathcal{D}_t = \{x_j\}_{j=1}^n$, whose underlying distribution differs from that of the source. 
Following the standard SFDA setting, the source data $\mathcal{D}_{s}$ is \emph{strictly inaccessible} during adaptation, due to data privacy or storage limitations.
We initialize our target model $\theta_t$ with the source parameters (i.e., $\theta_t = \theta_s$) and must bridge the domain gap using only the unlabeled samples in $\mathcal{D}_t$. The core challenge lies in recovering the target discriminative structure without the supervision of the source distribution or ground-truth target labels.

We further introduce a \emph{peer (counterpart) model} $\theta_v$ designed to provide \emph{complementary correctness patterns}. Specifically, the two models (i.e., $\theta_t$ and $\theta_v$) are assumed to make errors on different subsets of the target domain: there exist non-empty disagreement regions where only one of the two models is correct. 
This functional asymmetry enables \emph{mutual correction}, whereby each model can provide informative supervision to its peer on samples where its own prediction is correct and the other's is not. 
When applied iteratively, such cross-model supervision forms a positive reinforcing feedback loop, in which progressively refined predictions from one model serve as higher-fidelity pseudo-labels for the other, ultimately improving both models beyond their initial performance \cite{DeepMutualLearning_Zhang_cvpr, w2s_burns_openai, weak2strong_theory_lang_nips}.

\textbf{Dynamics of Mutual Correction.}~~
To formalize the dynamics of mutual correction, we partition the target data based on the correctness of $\theta_t$ and $\theta_v$. For sample $x$, let $y(x) \in \mathcal{Y}$ denote the (unknown) ground-truth label and $\hat{y}_i(x)$ the predicted label from model $\theta_i$ with $i \in \{t, v\}$.
We define the set of correct predictions for model $\theta_i$ as $\mathcal{C}_{i} =  \{x \in \mathcal{D}_t: \hat{y}_i(x) = y(x)\}$ and the set of incorrect predictions as its complement, $\overline{\mathcal{C}}_{i} = \mathcal{D}_{t}\setminus \mathcal{C}_{i}$. The intersection of these sets induces a partition of the target domain into four mutually exclusive (unobservable) subsets:
\begin{align} \label{eq:subsets_oracle}
    \mathcal{S}_{cc} = \mathcal{C}_{t} \cap \mathcal{C}_{v}, \quad
    \mathcal{S}_{c\bar{c}} = \mathcal{C}_{t} \cap \overline{\mathcal{C}}_{v}, \quad 
    \mathcal{S}_{\bar{c}c} = \overline{\mathcal{C}}_{t} \cap \mathcal{C}_{v}, \quad
    \mathcal{S}_{\bar{c}\bar{c}} = \overline{\mathcal{C}}_{t} \cap \overline{\mathcal{C}}_{v},
\end{align}
where each subset corresponds to a distinct joint correctness pattern of the two models.

Due to the absence of ground-truth labels in SFDA, mutual correction relies on each model using the other's predictions as pseudo-labels. This cross-model supervision induces two competing dynamics: \vspace{-2mm}
\begin{itemize}[leftmargin=1.5em]
    \item \textbf{Correction:}
    This occurs when a model receives a correct pseudo-label from its peer that can rectify its own error. Specifically, for samples in $\mathcal{S}_{\bar{c}c}$, the peer model $\theta_v$ provides correct supervision to $\theta_t$; symmetrically, $\theta_t$ can correct $\theta_v$ for samples in $\mathcal{S}_{c
    \bar{c}}$. Incorporating such signals directly improves the receiving model's discriminative power.    
    \item \textbf{Miscorrection:}
    This occurs when an incorrect prediction is used as supervision, potentially corrupting a model that is already correct. For instance, if $\theta_v$ supervises $\theta_t$ on $\mathcal{S}_{c\bar{c}}$ (or vice versa on $\mathcal{S}_{\bar{c}c}$), the peer model provides an erroneous pseudo-label. Accepting such signals propagates errors, leading to confirmation bias and a degradation of performance.   
\end{itemize}
The ideal strategy would accept supervision exclusively on \textit{correctly predicted} samples -- i.e., employing $\theta_v$ as a supervisor only on $\mathcal{S}_{\bar{c}c}$ and $\theta_t$ only on $\mathcal{S}_{c\bar{c}}$ -- thereby maximizing correction while avoiding miscorrection.
However, this \textit{oracle} selection is infeasible in practice: since ground-truth labels are \textit{not available}, there is no direct way to determine which model is correct for a given sample. The central challenge is therefore to approximate this selection in a \textit{fully unsupervised} manner, so that correction signals are amplified while miscorrection signals are suppressed. 

\subsection{When to Rely on Which Model: A Cut Statistic Approach}  \label{sec:observation}

To address the challenge of identifying correctly predicted samples without access to ground-truth labels, we employ the \emph{cut statistic}~\cite{cutstats_muhlenbach_jiis}, a graph-based measure that quantifies consistency of model predictions within a sample's local neighborhood. 
 
\begin{figure}[t]
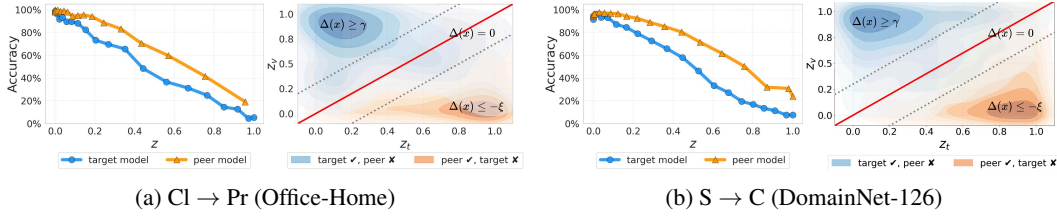

    \centering
    \begin{subfigure}[b]{0.49\textwidth}
        \centering
        \includegraphics[width=\textwidth]{Contents/Figures/Figure2_1.pdf}
        \caption{Cl $\rightarrow$ Pr (Office-Home)}
        \label{fig:obs-clpr}
    \end{subfigure}
    \hfill
    \begin{subfigure}[b]{0.49\textwidth}
        \centering
        \includegraphics[width=\textwidth]{Contents/Figures/Figure2_2.pdf}
        \caption{ S $\rightarrow$ C (DomainNet-126)}
        \label{fig:obs-rwpr}
    \end{subfigure}\vspace{-1mm}
    \caption{%
    \textbf{Cut statistic as a reliability proxy.}
    For each adaptation task, the left panel shows that prediction accuracy 
    decreases (approximately) monotonically with the cut statistic $z$ (computed via \eqref{eq:cut_estimate}) for both models. 
    The right panel shows the sample density in the $(z_t, z_v)$ plane for $\mathcal{S}_{c\bar{c}}$ (blue) and $\mathcal{S}_{\bar{c}c}$ (orange), which are concentrated in $\Delta(x)\ge \gamma_0$ and $\Delta(x)\le -\xi_0$, respectively.
}
    \label{fig:observation}\vspace{-1em}
\end{figure}

\textbf{Cut Statistic.}~~
For each target sample $x \in \mathcal{D}_t$, let $\mathcal{N}_{K}(x)$ denote its $K$-nearest neighbors. 
Then, the cut statistic for model $\theta_i$, where $i\in \{t, v\}$, is defined as:
\begin{equation} \label{eq:cut_original}
    z_{i}(x) = \!\! \sum_{x' \in \mathcal{N}_K(x)} \! \!\! w(x, x')  \cdot \mathbbm{1}[\hat{y}_{i}(x) \neq \hat{y}_{i}(x')]
\end{equation}
where $w(x, x')$ is a distance-based weighting function that assigns greater importance to closer neighbors. Intuitively, the cut statistic quantifies the \textit{local smoothness} of model predictions over the data graph. A low value indicates that a prediction is consistent with those of its neighbors, whereas a high value indicates that such local consistency does not hold. 

Under the cluster assumption~\cite{semi_chapelle_mit}, locally consistent predictions are generally more likely to be correct, suggesting that the cut statistic should be negatively correlated with predictive accuracy~\cite{cutstats_muhlenbach_jiis, weak2strong_theory_lang_nips}. However, in domain adaptation settings -- where distribution shift may distort the underlying structure -- this relationship is not guaranteed to hold. To examine whether this intuition remains valid, we empirically analyze the relationship between local consistency and prediction accuracy across diverse adaptation scenarios. Specifically, for each model $\theta_i$ with $i \in \{t, v\}$, we compute $z_{i}(x)$ for each target sample $x \in \mathcal{D}_{t}$, group samples into bins according to their cut values, and measure the average accuracy within each bin.
The results (Figure~\ref{fig:observation} left) reveal a clear and consistent trend across all evaluated settings, motivating the following observation (see Appendix~\ref{appendix:extended_obs} for additional results):

\vspace{1mm}
\begin{observation}\textit{(\textbf{Approximate Monotonicity})}
\label{obs:monotonicity}
\textit{
    For both $\theta_i$ with $i \in \{t, v\}$, prediction accuracy is approximately monotonically decreasing with respect to the cut statistic $z_i$.
}
\end{observation}
\vspace{-2mm}

\textbf{Connecting Correction Dynamics with Cut Statistics.}~~
Based on this observation, we assume that the conditional accuracy $\mathbb{P}\big(\hat{y}_i(x) = y(x) \mid z_i(x)\big)$ is monotonically decreasing in with respect to $z_i \in [0,1]$.\footnote{For notational simplicity, we will omit the dependence on $x$ in $y(x)$ and $\hat{y}_i(x)$ when it is clear from context.}
Since this monotonicity holds for both $\theta_t$ and $\theta_v$, the relative difference in their cut statistics can offer a principled, \textit{label-free} criterion for determining which model provides more reliable prediction on a per-sample basis.
To formalize this, we define the \textit{cut margin} $\Delta(x):= z_v(x) - z_t(x)$ to characterize the \textit{relative expertise} of the two models on sample $x$. The sign of $\Delta(x)$ indicates the direction of reliability: a positive value indicates that $\theta_t$ is more likely to be correct, while a negative value suggests the peer model $\theta_v$ is more reliable.
Crucially, the magnitude $|\Delta(x)|$ reflects the strength of this reliability gap, and thus measures the potential fidelity of the correction signal.
Building on this intuition, we approximate the unobservable oracle partitions $\mathcal{S}_{c\bar{c}}$ and $\mathcal{S}_{\bar{c}c}$ in \eqref{eq:subsets_oracle} by introducing margin thresholds $\gamma, \xi \in [0,1]$, identifying regions of the target domain where one model provides a clear reliability gap over the other:
\begin{equation} \label{eq:param_s}
    \hat{\mathcal{S}}_{c\bar{c}}(\gamma) = \{x \in \mathcal{D}_t : \Delta(x) \ge \gamma\}, \quad 
    \hat{\mathcal{S}}_{\bar{c}c}(\xi) = \{x \in \mathcal{D}_t :  \Delta(x) \le - \xi\} .
\end{equation}
Given different failure modes of $\theta_t$ and $\theta_v$, and recalling that the monotonic decrease of each model's accuracy with its own cut statistic (Observation~\ref{obs:monotonicity}), the model with higher local consistency is more likely to be correct and can thus provide high-fidelity pseudo-labels to supervise its less reliable counterpart. Notably, this intuition remains valid under approximate monotonicity; we provide theoretical bounds under noisy conditions in Appendix~\ref{app:noisy_scenario}. We formalize this intuition through the following baseline condition for effective mutual correction: \vspace{1mm}
\begin{proposition}[\textbf{Reliability-Separated Regions}]
\label{prop:threshold}
There exist baseline margins $\gamma_0, \xi_0$ such that:
\begin{align}
    \mathbb{P}\big(\hat{y}_t = y \mid x \in \hat{\mathcal{S}}_{c\bar{c}}(\gamma_0)\big) &> \mathbb{P}\big(\hat{y}_v = y \mid x \in \hat{\mathcal{S}}_{c\bar{c}}(\gamma_0)\big), \\
    \mathbb{P}\big(\hat{y}_v = y \mid x \in \hat{\mathcal{S}}_{\bar{c}c}(\xi_0)\big) &> \mathbb{P}\big(\hat{y}_t = y \mid x \in \hat{\mathcal{S}}_{\bar{c}c}(\xi_0)\big). 
\end{align}
Please refer to Appendix~\ref{app:proposition_clean} for the full proof of Proposition~\ref{prop:threshold}.
\end{proposition} \vspace{-3mm}
Here, $\gamma_0$ and $\xi_0$ denote the minimum reliability gaps required for one model’s prediction to serve as a trustworthy supervisory signal for its peer.

While Proposition~\ref{prop:threshold} shows that the supervising model is more accurate in regions where it is locally more consistent than its peer, effective mutual correction requires a stronger guarantee: cross-model supervision must produce a \textit{net-positive correction signal}, where corrections outnumber miscorrections. The following theorem proves that this dominance holds under the baseline margins:
\begin{theorem}[\textbf{Correction Dominance}]
\label{thm:correction}
For the baseline margins $\gamma_0, \xi_0$ satisfying Proposition~\ref{prop:threshold}:
\begin{align}
    \mathbb{P}\big(\hat{y}_t = y \land \hat{y}_v \ne y \mid x \in \hat{\mathcal{S}}_{c\bar{c}}(\gamma_0)\big) &>\mathbb{P}\big(\hat{y}_t \ne y \land \hat{y}_v = y \mid x \in \hat{\mathcal{S}}_{c\bar{c}}(\gamma_0)\big), \\
    \mathbb{P}\big(\hat{y}_v = y \land \hat{y}_t \ne y \mid x \in \hat{\mathcal{S}}_{\bar{c}c}(\xi_0)\big) &> \mathbb{P}\big(\hat{y}_v \ne y \land \hat{y}_t = y \mid x \in \hat{\mathcal{S}}_{\bar{c}c}(\xi_0)\big).
\end{align}
Please refer to Appendix~\ref{app:theorem_clean} for the full proof of Theorem~\ref{thm:correction}.
\end{theorem} \vspace{-3mm}
Figure~\ref{fig:observation} (right) illustrates this net-positive correction signal: as samples move toward the center of each mode in the $(z_t, z_v)$ plane, the density of corrections increases while that of miscorrections diminishes. This suggests that choosing sufficiently strict margins $\gamma_{\text{strict}}, \xi_{\text{strict}}$ beyond the baseline can simultaneously maximize corrections and minimize miscorrections. The following corollary formalizes this margin-driven trade-off:


\begin{corollary}[\textbf{Margin-Driven Strict Dominance}]
\label{corollary:strict_dominance}
For sufficiently strict margins $\gamma_{\text{strict}},\xi_{\text{strict}}$, the ratio of true correction to miscorrection increases compared to the baseline margins $\gamma_0,\xi_0$:
\begin{align}
    \frac{\mathbb{P}\big(\hat{y}_t = y \land \hat{y}_v \ne y \mid x \in \hat{\mathcal{S}}_{c\bar{c}}(\gamma_{\text{strict}})\big)}{\mathbb{P}\big(\hat{y}_t \ne y \land \hat{y}_v = y \mid x \in \hat{\mathcal{S}}_{c\bar{c}}(\gamma_{\text{strict}})\big)} &\ge \frac{\mathbb{P}\big(\hat{y}_t = y \land \hat{y}_v \ne y \mid x \in \hat{\mathcal{S}}_{c\bar{c}}(\gamma_0)\big)}{\mathbb{P}\big(\hat{y}_t \ne y \land \hat{y}_v = y \mid x \in \hat{\mathcal{S}}_{c\bar{c}}(\gamma_0)\big)} > 1.
\end{align}
An analogous result holds for $\hat{\mathcal{S}}_{\bar{c}c}(\cdot)$ with thresholds $\xi_0$ and $\xi_{\text{strict}}$.
\end{corollary}\vspace{-2mm}

\textbf{From Theory to Practice.}~~
Corollary~\ref{corollary:strict_dominance} shows that sufficiently enlarging the cut margin amplifies the dominance of true correction over miscorrection. In practice, however, selecting optimal thresholds $\gamma$ and $\xi$ is infeasible in a fully unlabeled setting. To bridge this gap, we replace hard thresholding with a soft gating mechanism that assigns supervision weights monotonically increasing in the cut margin $|\Delta(x)|$, prioritizing samples with a larger reliability gap. For simplicity, we adopt a symmetric formulation with $\gamma = \xi$, reducing the gating mechanism to a single margin-based criterion.

\section{Training with Reliability-Gated Supervision} \label{sec:method}
\subsection{Estimating Prediction Reliability via Cut Statistics} \label{sec:cut_computation}

To quantify prediction reliability via the cut statistic defined in \eqref{eq:cut_original}, we first characterize the local geometry of the target domain under both $\theta_t$ and $\theta_v$. Let $\phi_i \colon \mathcal{X} \to \mathbb{R}^{d_i}$ denote the feature extractor for model $i \in \{t, v\}$, mapping inputs into a latent space where pairwise distances are computed.\footnote{In our experiments, $\phi_i(x)$ corresponds to the output of the penultimate layer of each model (i.e., pre-logits).} Using these representations, we construct a local neighborhood graph by identifying the $K$-nearest neighbors $\mathcal{N}_K(x)$ for each sample $x$ under a standard distance metric $d$. To mitigate spurious connections -- particularly for samples located in low-density regions -- we restrict the neighborhood to \textit{mutual} neighbors, i.e., $\mathcal{N}_K^{\text{mut}}(x) = \{x' \mid x' \in \mathcal{N}_K(x) \land x \in \mathcal{N}_K(x')\}$.
Then, we construct a distance-weighted cut statistic for each model $i \in \{t, v\}$ as follows
\begin{equation} \label{eq:cut_estimate}
    \hat{z}_i(x) = \frac{\sum_{x' \in \mathcal{N}_K^{\text{mut}}(x)} \exp(-d(\phi_i(x), \phi_i(x'))) \cdot \mathbbm{1}[\hat{y}_i(x) \neq \hat{y}_i(x')]}{\sum_{x' \in \mathcal{N}_K^{\text{mut}}(x)} \exp(-d(\phi_i(x), \phi_i(x')))}.
\end{equation}
Here, $w(x, x') \propto \exp(-d(\phi_i(x), \phi_i(x')))$ and normalization ensures that cut statistics remain comparable across regions with different sample densities. Finally, to translate the cut statistic into an intuitive measure of prediction reliability, we apply an exponential mapping $r_i(x) = \exp(-\hat{z}_i(x)) \in (0, 1]$, ensuring that highly consistent samples receive scores near $1$.

Naively computing these scores requires an additional forward pass over the entire target data at every epoch, which is prohibitively expensive in SFDA settings where target domains often contain hundreds of thousands of samples. To circumvent this overhead, we reuse the latent features cached during the optimization forward pass of the preceding epoch. To further stabilize the estimates, we apply an exponential moving average across epochs, i.e., $\bar{r}_i^{(e)}(x) = m \cdot \bar{r}_i^{(e-1)}(x) + (1 - m) \cdot r_i^{(e)}(x)$, where $e$ denotes the epoch index and $m \in [0,1)$ is the momentum parameter.
For notational simplicity, we omit the epoch index hereafter and use $\bar{r}_i(x)$ in place of $r_i(x)$; the complete epoch-indexed formulation and algorithm details are provided in Appendix~\ref{appendix:epoch_indexed}. These reliability scores $\bar{r}_t$ and $\bar{r}_v$ serve as the foundation for determining how the two models supervise one another, which we will describe in the next subsection.

\subsection{Reliability-Gated Mutual Correction} \label{sec:mutual_correction}
Given the per-sample reliability estimated from Section~\ref{sec:cut_computation}, we now formulate a novel mutual correction objective where $\theta_t$ and $\theta_v$ iteratively refine each other by utilizing their predictions as pseudo-labels. The central challenge in this cross-model supervision is that we must dynamically determine the \textit{direction} and \textit{strength} of the supervision for each sample while \textit{preventing} reliable predictions from being corrupted by erroneous supervision from the less reliable peer.

\textbf{Margin-based Reliability Gating.}~~
Recall from Section~\ref{sec:observation} that there exists a margin threshold that identifies regions of the target domain where one model can reliably supervise the other. Rather than relying on such discrete thresholds, we adopt a continuous formulation in which the reliability gap directly governs both the \textit{direction} and \textit{strength} of supervision on a per-sample basis. Formally, we define the cross-model \textit{margin-based reliability gating} mechanism:
\begin{equation}\label{eq:gating}
    T_{v \to t}(x) = \sigma(\beta(\bar{r}_v(x)-\bar{r}_t(x))), \quad T_{t \to v}(x) = \sigma(\beta(\bar{r}_t(x)-\bar{r}_v(x)))
\end{equation}
where $\sigma$ denotes the sigmoid function. By the symmetry of the sigmoid, $T_{v \to t}(x) + T_{t \to v}(x) = 1$ holds for all $x$, ensuring that the total supervisory budget is conserved across the two directions. Here, $T_{v \to t}(x)$ represents the weight assigned to the supervisory signal from $\theta_v$ to $\theta_t$ (and vice versa).
The temperature parameter $\beta > 0$ controls the sharpness of the gating mechanism. As $\beta$ increases, the weights become more sensitive to the reliability gap $|\bar{r}_v(x) - \bar{r}_t(x)|$, approaching the hard-threshold behavior implied by Proposition~\ref{prop:threshold}. Conversely, smaller values of $\beta$ produce a smoother transition, allowing the model to remain less committed when the reliability difference is small (e.g., when both models are similarly uncertain or similarly confident).
By applying this gating directly to the mutual correction signal, we obtain the following cross-supervision loss:
\begin{equation}\label{eq:cross}
    \mathcal{L}_{\text{cross}}^{t}(x) = T_{v \to t}(x) \cdot D_{\mathrm{KL}}\bigl(\mathrm{sg}(p_v) \,\|\, p_t\bigr), \quad 
    \mathcal{L}_{\text{cross}}^{v}(x) = T_{t \to v}(x) \cdot D_{\mathrm{KL}}\bigl(\mathrm{sg}(p_t) \,\|\, p_v\bigr),
\end{equation}
where $\mathrm{sg}(\cdot)$ denotes the stop-gradient operator.
As a result, when one model has a clear reliability advantage on a given sample, it naturally assumes a stronger supervisory role, while the influence of the less reliable model is correspondingly attenuated.


\textbf{Temporal Self-Anchoring.}~~
While soft gating in \eqref{eq:cross} effectively mitigates miscorrection, it does not eliminate it entirely. Even small, non-zero gating weights from an unreliable (incorrect) peer can allow residual error signals that can accumulate over successive epochs, gradually degrading predictions that were initially correct. To prevent this, we further introduce a \textit{temporal self-anchoring} mechanism that encourages each model to remain consistent with its own previous predictions, dynamically weighted by its own reliability score:
\begin{align}\label{eq:temporal_anchor}
    \mathcal{L}_{\text{anc}}^{t}(x) = T_{t \to v}(x) \cdot D_{\mathrm{KL}}(\tilde{p}_t \,\|\, p_t), \quad
    \mathcal{L}_{\text{anc}}^{v}(x) = T_{v \to t}(x) \cdot D_{\mathrm{KL}}(\tilde{p}_v \,\|\, p_v)
\end{align}
where $\tilde{p}_t$ and $\tilde{p}_v$ denote the predictions from the previous epoch. Since $T_{t \to v}(x)$ is large only when $\theta_t$ is highly reliable, and symmetrically for $T_{v \to t}(x)$, the anchoring term selectively preserves confident predictions while allowing less reliable ones to be flexibly updated through cross-model supervision.
 
\textbf{Training Objective.}~~
Combining reliability-gated mutual correction in \eqref{eq:cross} with temporal self-anchoring in \eqref{eq:temporal_anchor}, we define the overall training objective for each model as:
\begin{align}
    \mathcal{L}_{t} &= \underbrace{\mathbb{E}_{x} \left[ T_{v \to t} \cdot D_{\mathrm{KL}}(\mathrm{sg}(p_v) \| p_t) + T_{t \to v} \cdot D_{\mathrm{KL}}(\tilde{p}_t \| p_t) \right]}_{\text{sample-wise supervision}} + \underbrace{\lambda_{\mathrm{IIC}} \mathcal{L}_{\mathrm{IIC}} - \lambda_{\mathrm{div}} H(\bar{p}_t)}_{\text{batch-wise regularization}} \\
    \mathcal{L}_{v} &= \underbrace{\mathbb{E}_{x} \left[ T_{t \to v} \cdot D_{\mathrm{KL}}(\mathrm{sg}(p_t) \| p_v) + T_{v \to t} \cdot D_{\mathrm{KL}}(\tilde{p}_v \| p_v) \right]}_{\text{sample-wise supervision}} + \underbrace{\lambda_{\mathrm{IIC}} \mathcal{L}_{\mathrm{IIC}} - \lambda_{\mathrm{div}} H(\bar{p}_v)}_{\text{batch-wise regularization}}
\end{align}
where $\bar{p}_{t}$ and $\bar{p}_{v}$ denote the mean predictions over the current batch. 
While the sample-wise supervision drives local, reliability-aware updates, we complement it with two batch-level regularizers that impose global structure. The IIC loss $\mathcal{L}_{\mathrm{IIC}}$~\cite{iic_ji_iccv} maximizes the mutual information between $p_t$ and $p_v$ aggregated over the batch, aligning the prediction distributions of the two different views. This is particularly valuable on samples where the reliability gap $|\bar{r}_t - \bar{r}_v|$ is small, where the cross-supervision signal becomes nearly directionless. In addition, the diversity regularizer~\cite{shot_liang_icml} maximizes the entropy of $\bar{p}_i$, preventing collapse to a few dominant classes.

\section{Experiments}\label{sec:experiments}
\subsection{Experimental Setup}
\textbf{Settings.}~~
We evaluate \proposed against state-of-the-art baselines on four benchmarks: \textbf{Office-31}~\cite{office31_saenko_eccv} and
\textbf{Office-Home}~\cite{officehome_venkateswara_cvpr} as small- and medium-scale datasets, and \textbf{DomainNet-126}~\cite{domainnet_saito_iccv} and \textbf{VisDA}~\cite{visda_peng_arxiv} as large-scale datasets (details in Appendix~\ref{app:datasets}). We adopt ResNet-50~\cite{resnet_he_cvpr} as the backbone for the first three datasets and ResNet-101 for VisDA, under a closed-set SFDA protocol. Following recent works~\cite{difo_tang_cvpr,prode_tang_iclr}, we use CLIP~\cite{clip_radford_icml} with ViT-B/32 as the peer model in most experiments, applying CoOp~\cite{coop_zhou_ijcv} for prompt learning with a frozen image encoder. Beyond this default setup, we additionally evaluate \proposed in two practical scenarios: 
(i) replacing CLIP with a non-CLIP peer to test peer generality, and
(ii) privacy-constrained medical settings using Camelyon17 from the
\textbf{WILDS} benchmark~\cite{wilds_sagawa_arxiv} and diabetic retinopathy
adaptation from EyePACS~\cite{cuadros_eyepacs_journal} to APTOS 2019.

\subsection{Performance Evaluation}
\begin{wraptable}{R}{0.58\textwidth}
    \centering
    \vspace{-30pt}
    \scriptsize
    \caption{Closed-set SFDA results (Acc, \%) on \textbf{Office-31}.} \vspace{-1mm}
    \label{tab:office31}
    \setlength{\tabcolsep}{0.8pt} 
    \renewcommand{\arraystretch}{1.2}
    \begin{tabularx}{\linewidth}{lc | *{6}{C} | C}
        \hline
        Method & Venue & \tiny{A$\to$D} & \tiny{A$\to$W} & \tiny{D$\to$A} & \tiny{D$\to$W} & \tiny{W$\to$A} & \tiny{W$\to$D} & \tiny{Avg.} \\
        \hline
        Source  & --     & 79.1 & 76.5 & 59.9 & 95.5 & 61.4 & 98.8 & 78.5 \\
        \hline
        SHOT    & ICML20 & 94.4 & 88.4 & 73.2 & 96.7 & 74.9 & 99.8 & 87.9{\textcolor{blue}{\tiny /88.6}} \\
        NRC     & NIPS21 & 92.4 & 94.1 & 75.6 & 97.2 & 74.6 & 98.6 & 88.7{\textcolor{blue}{\tiny /89.4}} \\
        AaD     & NIPS22 & 87.6 & 86.5 & 58.9 & 90.6 & 58.7 & 90.6 & 78.8{\textcolor{blue}{\tiny /89.9}} \\
        AdaCon  & CVPR22 & 87.8 & 86.3 & 71.3 & 96.7 & 67.2 & 97.6 & 84.5{\textcolor{blue}{\tiny /--}} \\
        CoWA    & ICML22 & 88.0 & 90.9 & 74.6 & 98.1 & 75.8 & 99.4 & 87.8{\textcolor{blue}{\tiny /90.3}} \\
        PLUE    & CVPR23 & 89.4 & 89.2 & 70.6 & 97.1 & 67.8 & 96.6 & 85.1{\textcolor{blue}{\tiny /--}} \\
        ELR     & ICLR23 & 92.2 & 92.2 & 76.5 & 97.4 & 76.1 & 99.6 & 89.0{\textcolor{blue}{\tiny /89.6}} \\
        DIFO-V  & CVPR24 & 96.2 & 95.0 & 83.2 & 96.1 & 83.0 & 99.2 & 92.1{\textcolor{blue}{\tiny /92.5}} \\
        ProDe-V & ICLR25 & 96.8 & \textbf{\textcolor{red}{96.6}} & 82.9 & 97.2 & 82.0 & 99.0 & 92.4{\textcolor{blue}{\tiny /92.6}} \\
        \hline
        \rowcolor{gray!35}
        \proposed    & NIPS26     & \textbf{\textcolor{red}{98.6}} & 95.2 & \textbf{\textcolor{red}{83.9}} & \textbf{\textcolor{red}{98.2}} & \textbf{\textcolor{red}{83.3}} & \textbf{\textcolor{red}{100.0}} & \textbf{\textcolor{red}{93.2}} \\
        \hline
    \end{tabularx}
    \vspace{-10pt}
\end{wraptable}
\begin{table*}[h] 
\centering
\scriptsize
\newcolumntype{A}{>{\centering\arraybackslash}p{0.95cm}}
\newcolumntype{M}{>{\raggedright\arraybackslash}p{0.9cm}}
\newcolumntype{V}{>{\centering\arraybackslash}p{0.9cm}}
\caption{Comparison on the Office-Home and VisDA datasets~(Acc, \%).} \vspace{-1.5mm}
\label{tab:officehome_visda}
\setlength{\tabcolsep}{1.8pt} 
\renewcommand{\arraystretch}{1.1}
\begin{tabularx}{\textwidth}{MV | *{12}{C} A | A} 
\toprule
\multirow{2}{*}{Method} & \multirow{2}{*}{Venue} & \multicolumn{13}{c|}{Office-Home} & VisDA \\
\cmidrule(lr){3-15} \cmidrule(l){16-16}
& & {\tiny Ar$\to$Cl} & {\tiny Ar$\to$Pr} & {\tiny Ar$\to$Rw} & {\tiny Cl$\to$Ar} & {\tiny Cl$\to$Pr} & {\tiny Cl$\to$Rw} & {\tiny Pr$\to$Ar} & {\tiny Pr$\to$Cl} & {\tiny Pr$\to$Rw} & {\tiny Rw$\to$Ar} & {\tiny Rw$\to$Cl} & {\tiny Rw$\to$Pr} & {\tiny Avg.} & {\tiny Sy$\to$Re} \\
\midrule
Source   & -- & 43.6 & 67.0 & 73.9 & 49.9 & 60.2 & 62.6 & 51.7 & 40.9 & 72.6 & 64.1 & 46.3 & 78.1 & 59.2 & 49.1 \\
\midrule
SHOT     & ICML20 & 55.2 & 78.5 & 81.3 & 68.9 & 79.0 & 78.8 & 68.2 & 54.3 & 81.7 & 73.5 & 59.1 & 83.6 & 71.8{\textcolor{blue}{\tiny /71.8}} & 82.3{\textcolor{blue}{\tiny /82.9}} \\
NRC      & NIPS21 & 54.7 & 77.1 & 79.4 & 64.7 & 76.2 & 75.3 & 63.3 & 53.6 & 78.8 & 70.2 & 56.5 & 82.5 & 69.4{\textcolor{blue}{\tiny /72.2}} & 83.4{\textcolor{blue}{\tiny /85.9}} \\
AaD      & NIPS22 & 57.7 & 78.5 & 80.8 & 68.0 & 79.6 & 78.9 & 66.5 & 56.6 & 81.6 & 70.9 & 57.3 & 83.7 & 71.7{\textcolor{blue}{\tiny /72.7}} & 78.8{\textcolor{blue}{\tiny /88.0}} \\
AdaCon   & CVPR22 & 49.5 & 76.1 & 77.0 & 63.5 & 73.2 & 74.4 & 60.7 & 47.1 & 77.9 & 67.8 & 49.7 & 80.3 & 66.4{\textcolor{blue}{\tiny /--}} & 67.2{\textcolor{blue}{\tiny /78.7}} \\
CoWA     & ICML22 & 56.7 & 79.0 & 80.9 & 68.6 & 78.7 & 79.8 & 68.1 & 55.6 & 82.6 & 72.4 & 61.7 & 83.6 & 72.3{\textcolor{blue}{\tiny /72.5}} & 84.3{\textcolor{blue}{\tiny /86.9}} \\
PLUE     & CVPR23 & 46.3 & 61.9 & 71.6 & 59.2 & 63.7 & 65.8 & 58.0 & 44.3 & 72.1 & 63.0 & 49.2 & 76.6 & 61.0{\textcolor{blue}{\tiny /--}} & 86.6{\textcolor{blue}{\tiny /90.0}} \\
ELR      & ICLR23 & 55.9 & 77.6 & 79.9 & 61.4 & 76.8 & 76.5 & 62.4 & 53.3 & 80.9 & 68.9 & 58.7 & 83.6 & 69.7{\textcolor{blue}{\tiny /72.6}} & 85.8{\textcolor{blue}{\tiny /85.8}} \\
DIFO-V   & CVPR24 & 70.7 & 90.5 & 89.2 & 82.4 & 90.2 & 89.2 & 80.6 & 71.1 & 89.4 & 82.9 & 71.0 & 91.0 & 83.2{\textcolor{blue}{\tiny /83.1}} & 89.8{\textcolor{blue}{\tiny /90.3}} \\
ProDe-V  & ICLR25 & 71.8 & 92.2 & 90.5 & 82.7 & 91.6 & 90.6 & 82.4 & 72.8 & 90.8 & 82.8 & 72.7 & 92.0 & 84.4{\textcolor{blue}{\tiny /84.5}} & 90.6{\textcolor{blue}{\tiny /91.0}} \\
\midrule
\rowcolor{gray!35}
\proposed     & NIPS26     & \textbf{\textcolor{red}{76.9}} & \textbf{\textcolor{red}{93.0}} & \textbf{\textcolor{red}{91.9}} & \textbf{\textcolor{red}{85.9}} & \textbf{\textcolor{red}{92.5}} & \textbf{\textcolor{red}{91.6}} & \textbf{\textcolor{red}{85.0}} & \textbf{\textcolor{red}{76.7}} & \textbf{\textcolor{red}{92.0}} & \textbf{\textcolor{red}{85.7}} & \textbf{\textcolor{red}{76.1}} & \textbf{\textcolor{red}{93.2}} & \textbf{\textcolor{red}{86.7}} & \textbf{\textcolor{red}{90.9}} \\
\bottomrule
\end{tabularx}
\vspace{0em}

\caption{Comparison on the DomainNet dataset (Acc, \%).} \vspace{-1.5mm}
\label{tab:domainnet}
\setlength{\tabcolsep}{2pt}
\renewcommand{\arraystretch}{1.1}
\begin{tabularx}{\textwidth}{lc | *{12}{C} A} 
\toprule
Method & Venue & \multicolumn{13}{c}{DomainNet} \\
\cmidrule(lr){3-15}
& & {\tiny C$\to$P} & {\tiny C$\to$R} & {\tiny C$\to$S} & {\tiny P$\to$C} & {\tiny P$\to$R} & {\tiny P$\to$S} & {\tiny R$\to$C} & {\tiny R$\to$P} & {\tiny R$\to$S} & {\tiny S$\to$C} & {\tiny S$\to$P} & {\tiny S$\to$R} & {\tiny Avg.} \\
\midrule
Source   & --     & 44.7 & 59.9 & 47.5 & 53.3 & 75.3 & 46.2 & 55.1 & 62.6 & 46.4 & 55.2 & 50.8 & 59.6 & 54.7 \\
\midrule
SHOT     & ICML20 & 63.1 & 78.6 & 60.1 & 68.4 & 80.7 & 62.0 & 68.0 & 68.3 & 58.7 & 70.8 & 64.9 & 77.7 & 68.4{\textcolor{blue}{\tiny /--}} \\
NRC      & NIPS21 & 60.0 & 75.7 & 55.4 & 59.6 & 80.8 & 58.0 & 60.0 & 67.3 & 54.3 & 66.9 & 65.0 & 77.7 & 65.1{\textcolor{blue}{\tiny /--}} \\
AaD      & NIPS22 & 62.4 & 79.1 & 59.3 & 67.1 & 81.6 & 62.6 & 69.7 & 66.7 & 58.7 & 71.1 & 63.6 & 79.7 & 68.5{\textcolor{blue}{\tiny /--}} \\
AdaCon   & CVPR22 & 56.9 & 64.1 & 51.9 & 57.5 & 64.9 & 52.9 & 60.7 & 64.4 & 51.5 & 62.7 & 60.1 & 61.1 & 59.0{\textcolor{blue}{\tiny /--}} \\
CoWA     & ICML22 & 64.7 & 80.6 & 63.0 & 70.1 & 80.6 & 64.3 & 70.5 & 69.5 & 63.1 & 71.7 & 67.3 & 80.0 & 70.5{\textcolor{blue}{\tiny /--}} \\
PLUE     & CVPR23 & 55.5 & 62.0 & 53.9 & 57.2 & 68.4 & 53.0 & 58.5 & 61.8 & 52.3 & 60.8 & 58.5 & 65.6 & 58.9{\textcolor{blue}{\tiny /--}} \\
ELR      & ICLR23 & 64.8 & 76.1 & 61.3 & 66.7 & 81.8 & 63.7 & 69.1 & 69.2 & 62.0 & 72.3 & 66.8 & 77.5 & 69.3{\textcolor{blue}{\tiny /79.4}} \\
DIFO-V   & CVPR24 & 76.5 & 87.8 & 75.4 & 78.8 & 87.7 & 75.7 & 80.0 & 77.3 & 75.6 & 79.2 & 76.4 & 87.8 & 79.9{\textcolor{blue}{\tiny /80.0}} \\
ProDe-V  & ICLR25 & 80.9 & 91.5 & 77.4 & 84.1 & 91.3 & 76.9 & 84.2 & 81.3 & 76.6 & 84.0 & 81.2 & 91.2 & 83.4{\textcolor{blue}{\tiny /85.0}} \\
\midrule
\rowcolor{gray!35}
\proposed     & NIPS26     & \textbf{\textcolor{red}{82.8}} & \textbf{\textcolor{red}{92.3}} & \textbf{\textcolor{red}{81.8}} & \textbf{\textcolor{red}{86.2}} & \textbf{\textcolor{red}{92.2}} & \textbf{\textcolor{red}{81.0}} & \textbf{\textcolor{red}{86.9}} & \textbf{\textcolor{red}{82.9}} & \textbf{\textcolor{red}{81.7}} & \textbf{\textcolor{red}{86.6}} & \textbf{\textcolor{red}{83.0}} & \textbf{\textcolor{red}{92.3}} & \textbf{\textcolor{red}{85.8}} \\
\bottomrule
\end{tabularx}
\vspace{-1em}
\end{table*}

\begin{table*}[h]
    \centering
    \scriptsize
    \caption{Mutual correction between two ResNet-50 models from different source domains (Office-Home, Acc, \%). Bef./Aft.: accuracy before/after adaptation.} \vspace{-1.5mm}
    \label{tab:symmetric}
    \renewcommand{\arraystretch}{1.3}
    \setlength{\tabcolsep}{1.5pt}
    \newcolumntype{N}{>{\centering\arraybackslash}p{0.38cm}}
    \newcolumntype{D}{>{\centering\arraybackslash}p{0.50cm}}
    \begin{tabular}{NN|DD|DD || NN|DD|DD || NN|DD|DD || NN|DD|DD}
        \toprule
        \multicolumn{6}{c|}{Target: \textbf{Ar}} & 
        \multicolumn{6}{c|}{Target: \textbf{Cl}} & 
        \multicolumn{6}{c|}{Target: \textbf{Pr}} & 
        \multicolumn{6}{c}{Target: \textbf{Rw}} \\
        \cmidrule(lr){1-6} \cmidrule(lr){7-12} \cmidrule(lr){13-18} \cmidrule(lr){19-24}
        & & \multicolumn{2}{c|}{$\theta_1$} & \multicolumn{2}{c||}{$\theta_2$} & 
        & & \multicolumn{2}{c|}{$\theta_1$} & \multicolumn{2}{c||}{$\theta_2$} & 
        & & \multicolumn{2}{c|}{$\theta_1$} & \multicolumn{2}{c||}{$\theta_2$} & 
        & & \multicolumn{2}{c|}{$\theta_1$} & \multicolumn{2}{c}{$\theta_2$} \\
        $\theta_1$ & $\theta_2$ & Bef. & Aft. & Bef. & Aft. & 
        $\theta_1$ & $\theta_2$ & Bef. & Aft. & Bef. & Aft. & 
        $\theta_1$ & $\theta_2$ & Bef. & Aft. & Bef. & Aft. & 
        $\theta_1$ & $\theta_2$ & Bef. & Aft. & Bef. & Aft. \\
        \midrule
        Cl & Pr & 49.9 & 72.2 & 51.7 & 72.2 & 
        Ar & Pr & 43.6 & 59.8 & 40.9 & 59.9 & 
        Ar & Cl & 67.0 & 81.3 & 60.2 & 81.2 & 
        Ar & Cl & 73.9 & 83.1 & 62.6 & 83.1 \\
        \rowcolor{gray!35}
        Cl & Rw & 49.9 & \textcolor{red}{\textbf{75.8}} & 64.2 & \textcolor{red}{\textbf{75.8}} & 
        Ar & Rw & 43.6 & \textcolor{red}{\textbf{61.7}} & 46.3 & \textcolor{red}{\textbf{61.6}} & 
        Cl & Rw & 60.2 & \textcolor{red}{\textbf{86.4}} & 78.1 & \textcolor{red}{\textbf{86.4}} & 
        Ar & Pr & 73.9 & \textcolor{red}{\textbf{84.3}} & 72.6 & \textcolor{red}{\textbf{84.2}} \\
        Pr & Rw & 51.7 & 74.4 & 64.2 & 74.4 & 
        Pr & Rw & 40.9 & 60.4 & 46.3 & 60.4 & 
        Ar & Rw & 67.0 & 83.0 & 78.1 & 83.0 & 
        Cl & Pr & 62.6 & 83.9 & 72.6 & 83.9 \\
        \bottomrule
    \end{tabular}
   \vspace{-3mm}
\end{table*}

In all benchmark experiments, we report the \textit{last-epoch performance}. For a fair comparison, the best results reported in the literature are highlighted in {\textcolor{blue}{blue}} alongside our reproduced baseline numbers (detailed per-task results are provided in Appendix~\ref{appendix:paper_best}). As shown in Table~\ref{tab:office31}, \proposed achieves an average accuracy of 93.2\% on Office-31, outperforming all reproduced baselines in 5 out of 6 adaptation tasks (+0.8\% over reproduced ProDe-V) and exceeding the best reported result by +0.6\%. Tables~\ref{tab:officehome_visda} and~\ref{tab:domainnet} demonstrate larger gains on Office-Home and DomainNet-126, with improvements of +2.3\% and +2.4\% over reproduced ProDe-V, and +2.2\% and +0.8\% over the paper-reported best results, respectively. On VisDA, our method remains competitive with the state of the art (90.9\% vs. 91.0\% paper-reported; +0.3\% over reproduced ProDe-V). We attribute these consistent improvements across benchmarks to the reliability-based gating mechanism, which filters out miscorrections and mitigates the confirmation bias inherent in naive cross-supervision.

\textbf{Peer Model Generality.}~~
To verify that \proposed does not rely on vision-language pretraining, we replace the CLIP peer model with a second ResNet-50 pretrained on a different source domain. Since the two models (denoted as $\theta_1$ and $\theta_2$) are trained on distinct source distributions, they naturally exhibit complementary failure modes on the target domain (i.e., the condition required by \proposed). As shown in Table~\ref{tab:symmetric}, both models improve consistently across all 12 source domain pairs, regardless of the initial performance gap between $\theta_1$ and $\theta_2$. Notably, the adapted performance surpasses non-CLIP SFDA baselines (e.g., SHOT, NRC, AaD, ELR) in most combinations, confirming that \proposed serves as a general-purpose correction framework applicable to diverse knowledge sources as long as they exhibit distinct failure modes with the source-pretrained model.

\subsection{Additional Analysis}
\begin{wraptable}{R}{0.55\textwidth}
    \centering
    \vspace{-10pt}
    \scriptsize
    \caption{Ablation and sensitivity analysis (Avg.\ Acc, \%).
    $\mathcal{L}_{\mathrm{c}}$: cross-supervision KL,
    $\mathcal{L}_{\mathrm{a}}$: temporal anchoring.}
    \label{tab:analysis}
    \renewcommand{\arraystretch}{1.1}
    \begin{minipage}[t]{0.52\linewidth}
        \setlength{\tabcolsep}{1pt}
        \textbf{(a) Component ablation}\\[2pt]
        \begin{tabularx}{\linewidth}{X | ccc}
            \hline
                 & O31 & OH & DN \\
            \hline
            $\mathcal{L}_{\mathrm{c}}\!=\!0, \mathcal{L}_{\mathrm{a}}\!=\!0$ & 91.2 & 81.8 & 72.1 \\
            $\mathcal{L}_{\mathrm{c}}\!=\!0, \mathcal{L}_{\mathrm{a}}\!=\!1$ & 91.2 & 83.3 & 78.7 \\
            $\mathcal{L}_{\mathrm{c}}\!=\!1, \mathcal{L}_{\mathrm{a}}\!=\!0$ & 92.4 & 86.0 & 83.6 \\
            $\mathcal{L}_{\mathrm{c}}\!=\!1, \mathcal{L}_{\mathrm{a}}\!=\!1$ & 92.1 & 86.2 & 84.3 \\
            \hline
            \rowcolor{gray!35}
            \textbf{{\proposed}} & \textbf{\textcolor{red}{93.2}} & \textbf{\textcolor{red}{86.7}} & \textbf{\textcolor{red}{85.8}} \\
            \hline
        \end{tabularx}
        \par\vspace{5pt}\noindent%
        \textbf{(c) CLIP adaptation}\\[2pt]
        \begin{tabularx}{\linewidth}{X | ccc}
            \hline
            & O31 & OH & DN \\
            \hline
            Zero-shot & 79.4 & 79.0 & 80.2 \\
            \rowcolor{gray!35}
            \hline
            \textbf{{Adapted}}   & \textbf{\textcolor{red}{92.8}} & \textbf{\textcolor{red}{85.8}} & \textbf{\textcolor{red}{84.0}} \\
            \hline
        \end{tabularx}
    \end{minipage}%
    \hfill
    \begin{minipage}[t]{0.44\linewidth}
        \setlength{\tabcolsep}{1pt}
        \textbf{(b) Variant}\\[2pt]
        \begin{tabularx}{\linewidth}{X | ccc}
            \hline
            Variant & O31 & OH & DN \\
            \hline
            Frozen $\theta_v$    & 89.0 & 84.8 & 84.6 \\
            Entropy gating       & 92.5 & 86.3 & 85.0 \\
            \rowcolor{gray!35}
            \textbf{{\proposed}} & \textbf{\textcolor{red}{93.2}} & \textbf{\textcolor{red}{86.7}} & \textbf{\textcolor{red}{85.8}} \\
            Oracle partition $S$ & 96.0 & 93.3 & 92.7 \\
            \hline
        \end{tabularx}
        \par\vspace{6pt}\noindent%
        \textbf{(d) $\beta$ sensitivity (OH)}\\[2pt]
        \begin{tabularx}{\linewidth}{C | C}
            \hline
            $\beta$ & OH \\
            \hline
            2  & 86.7 \\
            4  & 86.6 \\
            \rowcolor{gray!35}
            6  & \textbf{\textcolor{red}{86.7}} \\
            \hline
        \end{tabularx}
    \end{minipage}
    \vspace{-7pt}
\end{wraptable}

\textbf{Ablation Studies.}~~
We analyze the contribution of each component through ablation and sensitivity studies on Office-31, Office-Home, and DomainNet-126. Table~\ref{tab:analysis}(a) ablates each loss term. Here, $0/1$ denote variants of \proposed in which each term is disabled or enabled with uniform (constant) weighting, replacing our reliability-based gating. $\mathcal{L}_{\mathrm{c}}$ provides the primary performance gain, while $\mathcal{L}_{\mathrm{a}}$ improves stability by mitigating noise from pseudo-labels by anchoring to earlier predictions. Combining both with adaptive gating in \proposed achieves the best performance, highlighting their complementary roles and the importance of reliability-aware weighting. Table~\ref{tab:analysis}(b) compares key design variants: freezing $\theta_v$ degrades correction quality, and entropy-based gating underperforms cut-based gating, supporting the use of the cut statistic as a more reliable sample-level signal due to its incorporation of local neighborhood structure. The oracle partition provides an upper bound, indicating that further improvements could be achieved by better identifying relative expertise. Table~\ref{tab:analysis}(c) confirms that ViL adaptation is essential, as zero-shot CLIP significantly underperforms the adapted peer. Table~\ref{tab:analysis}(d) shows robustness to the gating sharpness parameter $\beta$, requiring minimum tuning (full ablation details are provided in Appendix~\ref{appendix:ablation_full}).

\textbf{Training Dynamics.}~~
As a diagnostic, we track ground-truth state transitions and per-group cut dynamics throughout adaptation (Figure~\ref{fig:state_change_and_cut}). The proportion of jointly correct samples $\mathcal{S}_{\mathrm{cc}}$ increases sharply from 56.6\% to 91.7\%, while the one-sided correct sets $\mathcal{S}_{\mathrm{c\bar{c}}}$ and $\mathcal{S}_{\mathrm{\bar{c}c}}$ nearly disappear, indicating effective bidirectional knowledge transfer. The jointly incorrect set $\mathcal{S}_{\mathrm{\bar{c}\bar{c}}}$ also decreases modestly (9.8\%$\to$7.1\%), suggesting gradual improvement through neighborhood generalization. To further assess the reliability signal, we group samples by their initial ground-truth state and track the per-group average cut value over training epochs: samples in $\mathcal{S}_{\mathrm{cc}}$ maintain consistently low cut values, reflecting stable neighborhoods, samples in the one-sided sets show a steady decrease as the models become aligned, and samples in $\mathcal{S}_{\mathrm{\bar{c}\bar{c}}}$ maintain consistent high cut values. This confirms that \proposed assigns low reliability to samples that are difficult to correct. See Appendix~\ref{appendix:state_transition} for additional results.

\begin{figure}[t]
    \centering
    \includegraphics[width=\linewidth]{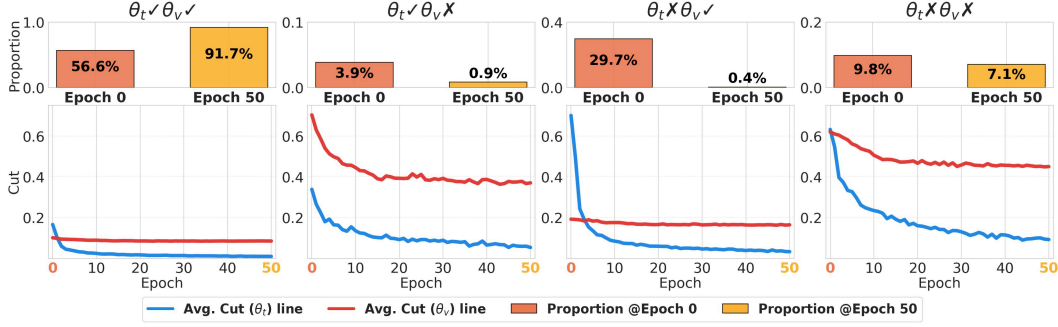}
    \caption{\textbf{Correctness pattern (top) and average cut trajectories (bottom) on Office-Home Cl$\to$Pr.}
    Samples are grouped by their initial state and tracked throughout adaptation.}
    \label{fig:state_change_and_cut}
    \vspace{-1em}
\end{figure}

\begin{wrapfigure}{r}{0.41\textwidth}
    \vspace{-5pt}
    \centering
    \includegraphics[width=\linewidth]{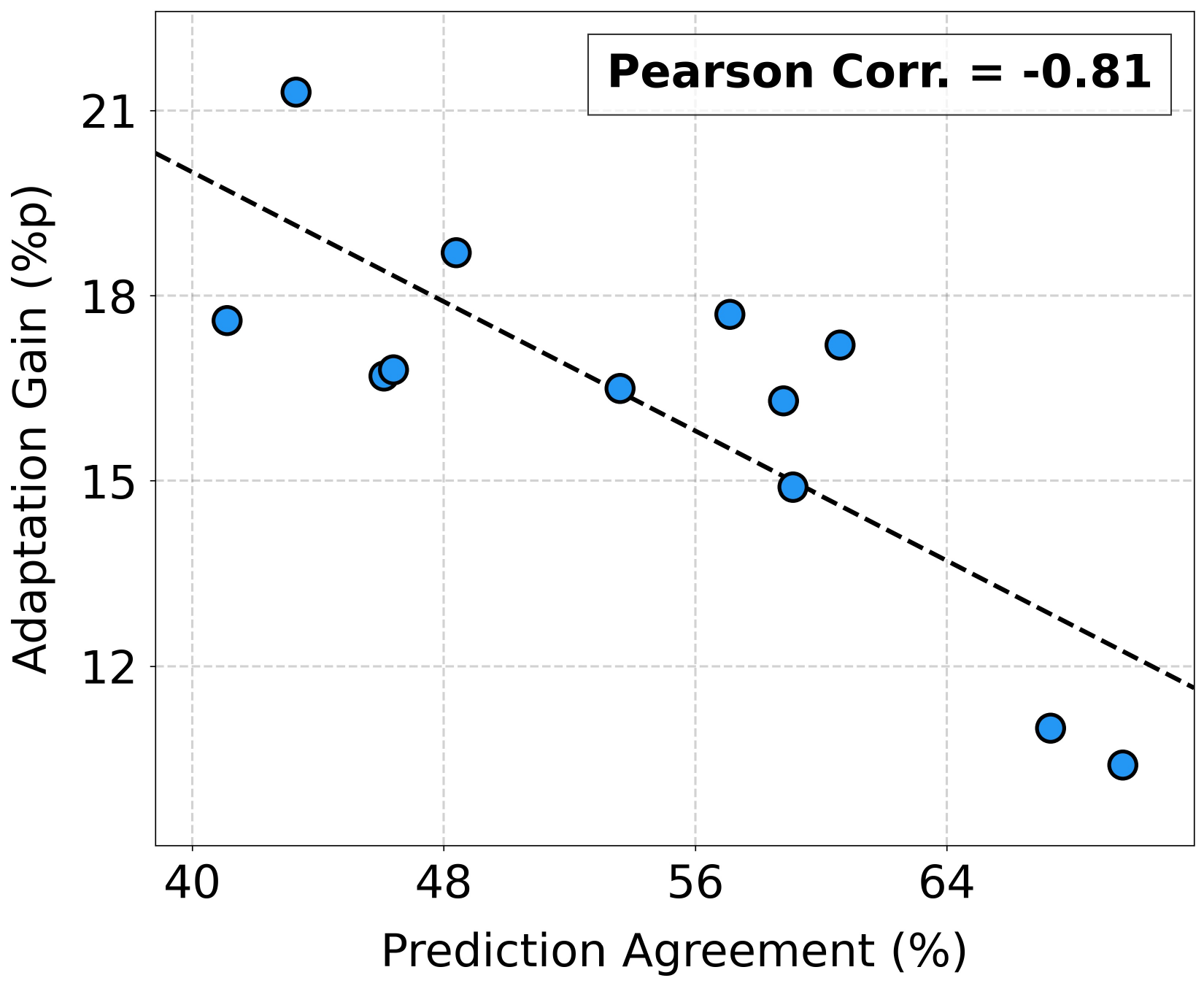}
    \caption{\textbf{Peer complementarity.}
    Prediction agreement vs.\ adaptation gain in the dual-ResNet setting of Table~\ref{tab:symmetric}.}
    \label{fig:peer_complementarity}
\end{wrapfigure}

\begingroup

\textbf{Effect of Peer Complementarity.}
To examine how peer complementarity affects mutual correction, we use the 12 dual-ResNet pairs in Table~\ref{tab:symmetric}. For each pair, we measure the \emph{pre-adaptation prediction agreement}, defined as the fraction of target
samples on which the two models predict the same class, and compare it with the \emph{adaptation gain}, defined as the average accuracy improvement of the two models after \proposed adaptation. As shown in Figure~\ref{fig:peer_complementarity}, higher prediction agreement is strongly associated with smaller adaptation gains (Pearson $r=-0.81$), supporting our motivation that complementary predictions provide more opportunities for cross-model correction. Nevertheless, the gain decreases gracefully as agreement increases: even at the highest observed agreement (69.6\%), SafeCut still yields a +10.4\%p improvement, with no degradation across the evaluated pairs. See Appendix~\ref{appendix:peer_complementarity} for detailed results.

\endgroup

\begin{wraptable}{R}{0.41\textwidth}
    \centering
    \scriptsize
    \caption{SFDA on Camelyon17 (Acc, \%).} \vspace{-1.5mm}
    \label{tab:camelyon17_sfda}
    \renewcommand{\arraystretch}{1.2}
    \setlength{\tabcolsep}{2pt}
    \begin{tabularx}{\linewidth}{l | C C}
        \toprule
        Method & Target & CLIP \\
        \midrule
        Source & 64.14 & 50.01 \\
        \midrule
        ProDe  & 76.27 & 50.00 \\
        $\mathcal{L}_{\mathrm{c}}\!=\!1, \mathcal{L}_{\mathrm{a}}\!=\!0$ & 77.29 & 77.20 \\
        $\mathcal{L}_{\mathrm{c}}\!=\!1, \mathcal{L}_{\mathrm{a}}\!=\!1$ & 87.96 & 87.78 \\
        \rowcolor{gray!35}
        \hline
        \textbf{\proposed} & \textbf{\textcolor{red}{94.96}} & \textbf{\textcolor{red}{94.20}} \\
        \bottomrule
    \end{tabularx}
    \vspace{-10pt}
\end{wraptable}
\textbf{Medical Domain Setup.}~~
Medical imaging is a canonical SFDA setting, where sharing source data is often restricted due to privacy concerns. We evaluate \proposed on Camelyon17 from WILDS benchmark~\cite{wilds_sagawa_arxiv}, training the source model on centers $\{0,3,4\}$, validating on $\{1\}$, and adapting to $\{2\}$. Unlike natural-image benchmarks, zero-shot CLIP transfers poorly to histopathology, even underperforming the source-pretrained target model ($50.01\%$ vs.\ $64.14\%$). In this setting, ProDe achieves only modest improvements on the target model and leaves (updated) CLIP similar to its zero-shot performance (50.00\%, with hyperparameters tuned via grid search; see Appendix~\ref{appendix:prode_camelyon}). This is because ProDe updates CLIP only through $\mathcal{L}_{\text{IIC}}$, providing no direct source-driven supervision; as a result, it fails to improve and continues to provide weak guidance. In contrast, \proposed trains both models through mutual, reliability-gated correction, enabling joint improvement. Using $\mathcal{L}_c$ alone matches ProDe, while adding $\mathcal{L}_a$ further stabilizes training. The complete model (\proposed) achieves $94.96\%$, significantly outperforming ProDe.  These results highlight the importance of reliability-aware mutual correction when the peer model is initially unreliable.

\begin{wraptable}{r}{0.43\textwidth}
    \vspace{-6pt}
    \centering
    \scriptsize

    \caption{EyePACS$\to$APTOS2019 (Acc, \%).}
    \label{tab:aptos_sfda}
    \vspace{-1.5mm}

    \renewcommand{\arraystretch}{1.15}
    \setlength{\tabcolsep}{3pt}

    \begin{tabularx}{\linewidth}{l | C C}
        \toprule
        Method & Target & CLIP \\
        \midrule
        Source & 39.0 & 25.4 \\
        \midrule
        ProDe & 35.2 & 35.3 \\
        \rowcolor{gray!35}
        \textbf{\proposed}
        & \textbf{\textcolor{red}{46.0}}
        & \textbf{\textcolor{red}{45.9}} \\
        \bottomrule
    \end{tabularx}
\end{wraptable}

To further evaluate this weak-peer regime, we additionally consider diabetic retinopathy severity classification, pretraining the source model on EyePACS~\cite{cuadros_eyepacs_journal} and adapting to APTOS 2019. As shown in Table~\ref{tab:aptos_sfda}, zero-shot CLIP achieves only $25.4\%$ accuracy on this 5-class task, indicating a particularly weak initial peer. ProDe reduces the target-model accuracy to $35.2\%$, whereas \proposed improves the target and peer models to $46.0\%$ and $45.9\%$, respectively. Together with the Camelyon17 results, these findings show that \proposed remains effective in specialized medical domains where pretrained ViL models transfer poorly, supporting the benefit of reliability-gated mutual correction beyond natural-image benchmarks.
\vspace{-1em}
\section{Conclusion} 
\label{sec:conclusion} 
\vspace{-1em}
In this paper, we presented \proposed, a novel SFDA framework driven by reliability-gated mutual correction. Unlike existing methods that unconditionally trust external models, we exploit the different failure modes of the target and peer models. Backed by solid theoretical grounding, \proposed uses the cut statistic as a label-free reliability proxy to dynamically control the direction and strength of cross-model supervision. Coupled with temporal self-anchoring, this mechanism amplifies valid corrective signals while mitigating miscorrections and suppressing error accumulation. Extensive experiments confirm that \proposed achieves state-of-the-art performance across multiple benchmarks. Crucially, \proposed not only generalizes across diverse peer models but also excels in challenging domains like medical imaging, demonstrating its versatility as a general-purpose correction framework.
\begin{ack}
We thank the NeurIPS reviewers for their valuable comments and suggestions. This work is supported by the National Research Foundation of Korea (NRF) grant funded by the Korea government (MSIT) (No. RS-2024-00358602) and the Institute of Information \& Communications Technology Planning \& Evaluation (IITP) grant funded by the Korea government (MSIT): Artificial Intelligence Graduate School Program (No. RS-2019-II190079, Korea University), Development of Unified Reasoning Technology Mimicking Human Cognition for Hierarchical Understanding and Unbounded Problem Solving (No. RS-2026-25522672), and Development of AI Co-Scientist based on Scientific Causal World Model (No.RS-2026-25507543).
\end{ack}

\bibliographystyle{unsrt}
\bibliography{bibliography}

\appendix
\clearpage
{\Huge\bfseries Appendix\par}
\vspace{2em}
{\Large\bfseries Table of Contents\par}
\vspace{0.3em}
\hrule
\vspace{1em}

\noindent\textbf{A\quad Proofs}\hfill\textbf{\pageref{app:proof}}\\[0.5em]
\hspace*{1.5em} A.1\quad Theoretical Justification under the Noisy Scenario\hfill\pageref{app:noisy_scenario}\\[1.0em]
\noindent\textbf{B\quad Evaluation Datasets}\hfill\textbf{\pageref{app:datasets}}\\[1.0em]
\noindent\textbf{C\quad Implementation Details}\hfill\textbf{\pageref{app:impl}}\\[0.5em]
\hspace*{1.5em} C.1\quad Epoch-Indexed Formulation\hfill\pageref{appendix:epoch_indexed}\\[1.0em]
\noindent\textbf{D\quad Additional Results}\hfill\textbf{\pageref{app:additional_results}}\\[0.5em]
\hspace*{1.5em} D.1\quad Comparison Against Paper-Reported Best Results\hfill\pageref{appendix:paper_best}\\[0.3em]
\hspace*{1.5em} D.2\quad VisDA Full Results\hfill\pageref{app:visda_full}\\[0.3em]
\hspace*{1.5em} D.3\quad Computational Analysis\hfill\pageref{app:compute}\\[0.3em]
\hspace*{1.5em} D.4\quad Extended Empirical Validation of Cut Statistic Properties\hfill\pageref{appendix:extended_obs}\\[0.3em]
\hspace*{1.5em} D.5\quad Extended Correctness Pattern Transition Results\hfill\pageref{appendix:state_transition}\\
[0.3em]
\hspace*{1.5em} D.6\quad Effect of Peer Complementarity\hfill\pageref{appendix:peer_complementarity}\\
[1.0em]
\noindent\textbf{E\quad Experimental Details}\hfill\textbf{\pageref{app:experimental details}}\\[0.5em]
\hspace*{1.5em} E.1\quad Implementation and Hyperparameter Details\hfill\pageref{appendix:hyperparams}\\[0.3em]
\hspace*{1.5em} E.2\quad ProDe Hyperparameter Search on Camelyon17\hfill\pageref{appendix:prode_camelyon}\\[0.3em]
\hspace*{1.5em} E.3\quad Full Per-Task Ablation Results\hfill\pageref{appendix:ablation_full}\\[1.0em]
\noindent\textbf{F\quad Limitation and Future Directions}\hfill\textbf{\pageref{appendix:limitations}}\\[1.0em]
\noindent\textbf{G\quad Broader Impacts}\hfill\textbf{\pageref{app:broader_impacts}}\\[1.0em]
\noindent\textbf{H\quad Asset Licenses and Terms of Use}\hfill\textbf{\pageref{app:asset_licenses}}\\[1.0em]

\vspace{0.5em}
\hrule
\clearpage

\section{Proofs}
\label{app:proof}

\begin{proposition}[\textbf{Reliability-Separated Regions}]
\label{app:proposition_clean}
There exist baseline margins $\gamma_0, \xi_0$ such that:
\begin{align}
    \mathbb{P}\big(\hat{y}_t = y \mid x \in \hat{\mathcal{S}}_{c\bar{c}}(\gamma_0)\big) &> \mathbb{P}\big(\hat{y}_v = y \mid x \in \hat{\mathcal{S}}_{c\bar{c}}(\gamma_0)\big), \\
    \mathbb{P}\big(\hat{y}_v = y \mid x \in \hat{\mathcal{S}}_{\bar{c}c}(\xi_0)\big) &> \mathbb{P}\big(\hat{y}_t = y \mid x \in \hat{\mathcal{S}}_{\bar{c}c}(\xi_0)\big). 
\end{align}
\end{proposition}

\begin{proof}[Proof]
Motivated by the approximate monotonicity observed in Observation~\ref{obs:monotonicity}, we consider a sufficiently separated positive-margin region and let $C := \{x \in \hat{\mathcal{S}}_{c\bar{c}}(\gamma_0)\}$. By definition, any sample in this region satisfies $\Delta(x) = z_v(x) - z_t(x) \ge \gamma_0 > 0$.

Using the law of total probability, the reliability gap between the two models over region $C$ can be expressed as an expectation (See detail proof in noisy version~\ref{app:prop_noise}):
\begin{equation*}
    \mathbb{P}\big(\hat{y}_t = y \mid C\big) - \mathbb{P}\big(\hat{y}_v = y \mid C\big) = \mathbb{E}_{x \mid C}\Big[\mathbb{P}\big(\hat{y}_t = y \mid z_t(x)\big) - \mathbb{P}\big(\hat{y}_v = y \mid z_v(x)\big)\Big].
\end{equation*}
For any $x \in C$, we have $\Delta(x) \ge \gamma_0 > 0$, and hence $z_t(x) < z_v(x)$. Under the cross-model reliability ordering induced by the cut statistic in this reliability-separated region, a lower cut statistic corresponds to a higher probability of correctness; thus, $\mathbb{P}(\hat y_t = y \mid z_t(x)) > \mathbb{P}(\hat y_v = y \mid z_v(x))$,
implying that the integrand is strictly positive for all $x \in C$.
Therefore, it follows that:
\begin{equation*}
    \mathbb{P}\big(\hat{y}_t = y \mid x \in \hat{\mathcal{S}}_{c\bar{c}}(\gamma_0)\big) - \mathbb{P}\big(\hat{y}_v = y \mid x \in \hat{\mathcal{S}}_{c\bar{c}}(\gamma_0)\big) > 0.
\end{equation*}
An analogous derivation confirms the symmetric result for $\hat{\mathcal{S}}_{\bar{c}c}(\xi_0)$.
\end{proof}

\begin{theorem}[\textbf{Correction Dominance}]
\label{app:theorem_clean}
For the baseline margins $\gamma_0, \xi_0$ satisfying Proposition~\ref{app:proposition_clean}:
\begin{align}
    \mathbb{P}\big(\hat{y}_t = y \land \hat{y}_v \ne y \mid x \in \hat{\mathcal{S}}_{c\bar{c}}(\gamma_0)\big) &> \mathbb{P}\big(\hat{y}_t \ne y \land \hat{y}_v = y \mid x \in \hat{\mathcal{S}}_{c\bar{c}}(\gamma_0)\big), \\
    \mathbb{P}\big(\hat{y}_v = y \land \hat{y}_t \ne y \mid x \in \hat{\mathcal{S}}_{\bar{c}c}(\xi_0)\big) &> \mathbb{P}\big(\hat{y}_v \ne y \land \hat{y}_t = y \mid x \in \hat{\mathcal{S}}_{\bar{c}c}(\xi_0)\big).
\end{align}
\end{theorem}

\begin{proof}
Let $E_1$ be the event $\{\hat{y}_t = y\}$ and $E_2$ be the event $\{\hat{y}_v = y\}$. For any conditional event $C$, the law of total probability yields:
\begin{align*}
    \mathbb{P}(E_1 \mid C) &= \mathbb{P}(E_1 \land E_2 \mid C) + \mathbb{P}(E_1 \land  E_2^c \mid C), \\
    \mathbb{P}(E_2 \mid C) &= \mathbb{P}(E_1 \land E_2 \mid C) + \mathbb{P}(E_1^c \land E_2 \mid C).
\end{align*}
Subtracting the second equation from the first, we establish the following fundamental identity:
\begin{equation}
    \mathbb{P}(E_1 \mid C) - \mathbb{P}(E_2 \mid C) = \mathbb{P}(E_1 \land E_2^c \mid C) - \mathbb{P}(E_1^c \land E_2 \mid C). \label{eq:prob_identity}
\end{equation}

For the first inequality, let the condition $C$ be $x \in \hat{\mathcal{S}}_{c\bar{c}}(\gamma_0)$. 
By Proposition~\ref{app:proposition_clean}, we have $\mathbb{P}(E_1 \mid C) > \mathbb{P}(E_2 \mid C)$. 
Applying this to Equation~\eqref{eq:prob_identity}, it directly follows that:
\begin{equation*}
    \mathbb{P}(E_1 \land E_2^c \mid C) - \mathbb{P}(E_1^c \land E_2 \mid C) > 0.
\end{equation*}
Substituting the events back to their original notation, we obtain:
\begin{equation*}
    \mathbb{P}\big(\hat{y}_t = y \land \hat{y}_v \ne y \mid x \in \hat{\mathcal{S}}_{c\bar{c}}(\gamma_0)\big) > \mathbb{P}\big(\hat{y}_t \ne y \land \hat{y}_v = y \mid x \in \hat{\mathcal{S}}_{c\bar{c}}(\gamma_0)\big),
\end{equation*}
which satisfies the first claim of the theorem.

Similarly, for the second inequality, let the condition $C'$ be $x \in \hat{\mathcal{S}}_{\bar{c}c}(\xi_0)$. 
Proposition~\ref{app:proposition_clean} states that $\mathbb{P}(E_2 \mid C') > \mathbb{P}(E_1 \mid C')$. 
Applying the same probability identity for condition $C'$, we get:
\begin{equation*}
    \mathbb{P}(E_2 \mid C') - \mathbb{P}(E_1 \mid C') = \mathbb{P}(E_1^c \land E_2 \mid C') - \mathbb{P}(E_1 \land E_2^c \mid C') > 0.
\end{equation*}
Therefore, we deduce:
\begin{equation*}
    \mathbb{P}\big(\hat{y}_v = y \land \hat{y}_t \ne y \mid x \in \hat{\mathcal{S}}_{\bar{c}c}(\xi_0)\big) > \mathbb{P}\big(\hat{y}_v \ne y \land \hat{y}_t = y \mid x \in \hat{\mathcal{S}}_{\bar{c}c}(\xi_0)\big),
\end{equation*}
\end{proof}

\subsection{Theoretical Justification under the Noisy Scenario}
\label{app:noisy_scenario}

While Section~\ref{sec:observation} assumes a strict monotonic relationship between the cut statistic and prediction accuracy, real-world distribution shifts in the target domain inevitably introduce noise. In an ideal, noise-free scenario, the difference in predictive accuracy between the two models would scale linearly with the cut margin on a per-sample basis. That is, within the valid bounds of probability, the accuracy difference would be proportional to the cut margin at a rate of η. In practice, however, individual samples exhibit empirical deviations from this ideal relationship. To account for such sample-wise variance, we introduce a noise tolerance term ϵ to establish a rigorous lower bound. Here, we formalize this relaxed condition as \textit{noisy monotonicity} and prove that correction dominance is still theoretically guaranteed provided the margin is sufficiently large.

\begin{assumption}[\textbf{Noisy Monotonicity}]
\label{app:assumption_noise}
For any target sample $x \in \mathcal{D}_t$ with cut statistics $z_t(x)$ and $z_v(x)$, there exist a decay rate $\eta > 0$ and noise level $\epsilon \ge 0$ such that:
\begin{equation}
    \mathbb{P}\big(\hat{y}_t=y \mid z_t(x)\big) - \mathbb{P}\big(\hat{y}_v=y \mid z_v(x)\big) \ge \eta \Delta(x) - \epsilon,
\end{equation}
where $\Delta(x) = z_v(x) - z_t(x)$ is the cut margin.
\end{assumption}

Based on this assumption, we can establish the lower bound for the reliability gap in the margin-separated regions.

\begin{proposition}[\textbf{Reliability Gap under Noise}]
\label{app:prop_noise}
Under Assumption~\ref{app:assumption_noise}, for the margin-separated regions $\hat{\mathcal{S}}_{c\bar{c}}(\gamma)$ and $\hat{\mathcal{S}}_{\bar{c}c}(\xi)$, the reliability gap between $\theta_t$ and $\theta_v$ is bounded by:
\begin{align}
    \mathbb{P}\big(\hat{y}_t=y \mid x \in \hat{\mathcal{S}}_{c\bar{c}}(\gamma)\big) - \mathbb{P}\big(\hat{y}_v=y \mid x \in \hat{\mathcal{S}}_{c\bar{c}}(\gamma)\big) &\ge \eta\gamma - \epsilon, \\
    \mathbb{P}\big(\hat{y}_v=y \mid x \in \hat{\mathcal{S}}_{\bar{c}c}(\xi)\big) - \mathbb{P}\big(\hat{y}_t=y \mid x \in \hat{\mathcal{S}}_{\bar{c}c}(\xi)\big) &\ge \eta\xi - \epsilon.
\end{align}
\end{proposition}

\begin{proof}
Let $C$ denote the event $\{x \in \hat{\mathcal{S}}_{c\bar{c}}(\gamma)\}$. By the definition of this region in Equation~\eqref{eq:param_s}, any sample conditioned on $C$ strictly satisfies $\Delta(x) \ge \gamma$.

Using the definition of conditional probability, the probability that model $\theta_t$ is correct given $C$ is:
\begin{equation*}
    \mathbb{P}(\hat{y}_t=y \mid C) = \int_{x \in \mathcal{D}_t} \mathbb{P}(\hat{y}_t=y \mid x, C) \cdot \mathbb{P}(x \mid C) dx.
\end{equation*}
Since $\mathbb{P}(x \mid C) = 0$ for any $x \notin C$, we can restrict the integration domain to $C$. Furthermore, assuming the prediction correctness primarily depends on the cut statistic $z_t(x)$, we can express this as an expectation over the region $C$:
\begin{equation*}
    \mathbb{P}(\hat{y}_t=y \mid C) = \int_{x \in C} \mathbb{P}\big(\hat{y}_t=y \mid z_t(x)\big) \cdot \mathbb{P}(x \mid C) dx = \mathbb{E}_{x \mid C}\big[\mathbb{P}\big(\hat{y}_t=y \mid z_t(x)\big)\big].
\end{equation*}
Applying the same logic for $\hat{y}_v=y$ and subtracting the two yields:
\begin{equation*}
    \mathbb{P}(\hat{y}_t=y \mid C) - \mathbb{P}(\hat{y}_v=y \mid C) = \mathbb{E}_{x \mid C}\Big[\mathbb{P}\big(\hat{y}_t=y \mid z_t(x)\big) - \mathbb{P}\big(\hat{y}_v=y \mid z_v(x)\big)\Big].
\end{equation*}
By Assumption~\ref{app:assumption_noise}, the integrand is lower-bounded by $\eta \Delta(x) - \epsilon$. Since $\Delta(x) \ge \gamma$ for all $x \in C$, we have:
\begin{equation*}
    \mathbb{E}_{x \mid C}\Big[\mathbb{P}\big(\hat{y}_t=y \mid z_t(x)\big) - \mathbb{P}\big(\hat{y}_v=y \mid z_v(x)\big)\Big] \ge \mathbb{E}_{x \mid C}\big[\eta\gamma - \epsilon\big] = \eta\gamma - \epsilon.
\end{equation*}
Therefore, $\mathbb{P}\big(\hat{y}_t=y \mid C\big) - \mathbb{P}\big(\hat{y}_v=y \mid C\big) \ge \eta\gamma - \epsilon$.
\end{proof}
For $\hat{\mathcal{S}}_{\bar{c}c}(\xi)$, applying the symmetric noisy-monotonicity relation $\mathbb{P}(\hat y_v=y\mid z_v(x)) -\mathbb{P}(\hat y_t=y\mid z_t(x)) \ge \eta(-\Delta(x))-\epsilon$ together with $\Delta(x)\le-\xi$ yields the second inequality analogously.
Building upon this reliability gap, we now establish the condition under which true correction strictly dominates miscorrection in the presence of noise.
\newpage
\begin{theorem}[\textbf{Correction Dominance under Noise}]
\label{app:theorem_noise}
For the margin-separated region $\hat{\mathcal{S}}_{c\bar{c}}(\gamma)$ satisfying the condition $\gamma > \frac{\epsilon}{\eta}$, the probability of true correction strictly dominates miscorrection by at least the margin $\eta\gamma - \epsilon$:
\begin{equation}
    \mathbb{P}\big(\hat{y}_t = y \land \hat{y}_v \ne y \mid x \in \hat{\mathcal{S}}_{c\bar{c}}(\gamma)\big) - \mathbb{P}\big(\hat{y}_t \ne y \land \hat{y}_v = y \mid x \in \hat{\mathcal{S}}_{c\bar{c}}(\gamma)\big) \ge \eta\gamma - \epsilon > 0.
\end{equation}
Symmetrically, for $\hat{\mathcal{S}}_{\bar{c}c}(\xi)$ satisfying $\xi > \frac{\epsilon}{\eta}$:
\begin{equation}
    \mathbb{P}\big(\hat{y}_v = y \land \hat{y}_t \ne y \mid x \in \hat{\mathcal{S}}_{\bar{c}c}(\xi)\big) - \mathbb{P}\big(\hat{y}_v \ne y \land \hat{y}_t = y \mid x \in \hat{\mathcal{S}}_{\bar{c}c}(\xi)\big) \ge \eta\xi - \epsilon > 0.
\end{equation}
\end{theorem}

\begin{proof}
Let $E_1$ be the event $\{\hat{y}_t = y\}$ and $E_2$ be the event $\{\hat{y}_v = y\}$. For any conditional event $C$, the law of total probability yields:
\begin{align*}
    \mathbb{P}(E_1 \mid C) &= \mathbb{P}(E_1 \land E_2 \mid C) + \mathbb{P}(E_1 \land  E_2^c \mid C), \\
    \mathbb{P}(E_2 \mid C) &= \mathbb{P}(E_1 \land E_2 \mid C) + \mathbb{P}(E_1^c \land E_2 \mid C).
\end{align*}
Subtracting the second equation from the first, we establish the following fundamental identity:
\begin{equation}
    \mathbb{P}(E_1 \mid C) - \mathbb{P}(E_2 \mid C) = \mathbb{P}(E_1 \land E_2^c \mid C) - \mathbb{P}(E_1^c \land E_2 \mid C). \label{eq:prob_identity_noise}
\end{equation}

For the first inequality, let the condition $C$ be $x \in \hat{\mathcal{S}}_{c\bar{c}}(\gamma)$. 
By Proposition~\ref{app:prop_noise}, we have:
\begin{equation*}
    \mathbb{P}(E_1 \mid C) - \mathbb{P}(E_2 \mid C) \ge \eta\gamma - \epsilon.
\end{equation*}
Applying this to Equation~\eqref{eq:prob_identity_noise}, it directly follows that:
\begin{equation*}
    \mathbb{P}(E_1 \land E_2^c \mid C) - \mathbb{P}(E_1^c \land E_2 \mid C) \ge \eta\gamma - \epsilon.
\end{equation*}
Given the condition $\gamma > \frac{\epsilon}{\eta}$, we know that $\eta\gamma - \epsilon > 0$. 
Substituting the events back to their original notation, we obtain:
\begin{equation*}
    \mathbb{P}\big(\hat{y}_t = y \land \hat{y}_v \ne y \mid x \in \hat{\mathcal{S}}_{c\bar{c}}(\gamma)\big) - \mathbb{P}\big(\hat{y}_t \ne y \land \hat{y}_v = y \mid x \in \hat{\mathcal{S}}_{c\bar{c}}(\gamma)\big) \ge \eta\gamma - \epsilon > 0,
\end{equation*}
which satisfies the first claim of the theorem.

Similarly, for the second inequality, let the condition $C'$ be $x \in \hat{\mathcal{S}}_{\bar{c}c}(\xi)$. 
Proposition~\ref{app:prop_noise} states that $\mathbb{P}(E_2 \mid C') - \mathbb{P}(E_1 \mid C') \ge \eta\xi - \epsilon$. 
Applying the same probability identity for condition $C'$, we get:
\begin{equation*}
    \mathbb{P}(E_2 \mid C') - \mathbb{P}(E_1 \mid C') = \mathbb{P}(E_1^c \land E_2 \mid C') - \mathbb{P}(E_1 \land E_2^c \mid C') \ge \eta\xi - \epsilon.
\end{equation*}
Given $\xi > \frac{\epsilon}{\eta}$, we deduce:
\begin{equation*}
    \mathbb{P}\big(\hat{y}_v = y \land \hat{y}_t \ne y \mid x \in \hat{\mathcal{S}}_{\bar{c}c}(\xi)\big) - \mathbb{P}\big(\hat{y}_v \ne y \land \hat{y}_t = y \mid x \in \hat{\mathcal{S}}_{\bar{c}c}(\xi)\big) \ge \eta\xi - \epsilon > 0.
\end{equation*}
\end{proof}


\section{Evaluation Datasets}
\label{app:datasets}

In this paper, our method is evaluated on five widely used benchmarks for domain adaptation as follows.

\begin{itemize}[leftmargin=1.5em]
    \item \textbf{Office-31}~\citep{office31_saenko_eccv} is a small-scale dataset consisting of three domains: Amazon~(A), DSLR~(D), and Webcam~(W), all featuring real-world objects in office environments. The Amazon domain contains images downloaded from \texttt{amazon.com} under studio lighting conditions, whilst the DSLR and Webcam domains consist of images captured with a high-resolution SLR camera and a low-resolution webcam, respectively. The dataset contains 4,652 images across 31 categories in total.

    \item \textbf{Office-Home}~\citep{officehome_venkateswara_cvpr} is a medium-scale dataset containing approximately 15,500 images of 65 everyday object categories across four visually distinct domains: Artistic images~(Ar), Clip Art~(Cl), Product images~(Pr), and Real-World images~(Rw). The dataset has on average 70 images per category and is notably more challenging than Office-31 due to its larger number of categories and greater domain diversity.

    \item \textbf{VisDA-C}~\citep{visda_peng_arxiv} is a large-scale dataset designed for the simulation-to-reality domain shift in object classification. The source domain comprises 152,397 synthetic images rendered from 3D CAD models under varying viewpoints and lighting conditions, whilst the target validation domain consists of 55,388 real object images cropped from the Microsoft COCO dataset. The dataset spans 12 object categories and poses a substantial distribution gap due to the synthetic-to-real nature of the shift.

    \item \textbf{DomainNet-126}~\citep{domainnet_saito_iccv} is a large-scale and challenging benchmark derived from the original DomainNet dataset~\citep{momentmatching_peng_iccv}, which contains approximately 600k images of 345 classes across 6 domains. Following~\citep{momentmatching_peng_iccv}, we use a curated subset of 4 domains---Clipart~(C), Painting~(P), Real~(R), and Sketch~(S)---and 126 classes with reduced label noise, resulting in roughly 145k images in total.
    
    \item \textbf{Camelyon17-WILDS}~\citep{wilds_sagawa_arxiv} is a 
    histopathology benchmark for evaluating robustness to domain shift across hospitals. Each input is a $96 \times 96$ tissue patch, and the task is binary classification of whether the patch contains tumor tissue. The 5 hospitals serve as distinct domains, 
    with models trained on 3 hospitals and evaluated on a held-out hospital, yielding 455,954 labeled examples in total.
    
    \item \textbf{EyePACS}~\citep{cuadros_eyepacs_journal} is a retinal fundus image dataset for diabetic retinopathy screening and severity assessment. It contains fundus photographs acquired under diverse imaging conditions, resulting in substantial variation in image appearance and quality. We use EyePACS as the source domain.

    \item \textbf{APTOS 2019}\footnote{ \url{https://www.kaggle.com/competitions/aptos2019-blindness-detection/data}} is a retinal fundus image dataset for diabetic retinopathy severity classification. The task consists of five classes corresponding to increasing levels of diabetic retinopathy severity. We use APTOS 2019 as the target domain.
\end{itemize}

\section{Implementation Details}
\label{app:impl}

\subsection{Epoch-Indexed Formulation}
\label{appendix:epoch_indexed}

The main text omits epoch indices for notational brevity; here we make the full temporal structure explicit and describe the efficient caching scheme that underlies the implementation.

\paragraph{Feature Caching.}
Computing cut statistics requires the complete feature set $\{\phi_i(x_j)\}_{j=1}^{n}$ over the entire target domain, as nearest-neighbor relationships cannot be resolved within a single mini-batch.
A naïve implementation would therefore require a dedicated full forward pass at every epoch solely for cut statistic computation, incurring significant overhead when the target domain contains hundreds of thousands of samples.

To eliminate this cost, we maintain a feature bank $\mathcal{B}_i^{\phi}$ and a prediction bank $\mathcal{B}_i^{p}$ for each model $i \in \{t, v\}$, populated \emph{in-place} during the standard optimization forward pass:
\begin{equation}
    \mathcal{B}_i^{\phi,(e)}[j] \;\leftarrow\; \phi_i^{(e)}(x_j), \qquad
    \mathcal{B}_i^{p,(e)}[j]   \;\leftarrow\; p_i^{(e)}(x_j), \qquad j \in \mathcal{B},
\end{equation}
so that by the end of epoch $e$ the banks hold the most recent representations without any additional forward pass.
Because the CLIP image encoder $\phi_v$ is entirely frozen, $\mathcal{B}_v^{\phi}$ is computed once during warmup and never modified thereafter.
At the end of each epoch, the epoch-indexed cut statistic is computed from the accumulated banks:
\begin{equation}
    \hat{z}_i^{(e)}(x) = \frac{
        \displaystyle\sum_{x' \in \mathcal{N}_K^{\mathrm{mut},(e)}(x)}
        \exp\!\bigl(-d(\mathcal{B}_i^{\phi,(e)}(x),\, \mathcal{B}_i^{\phi,(e)}(x'))\bigr)
        \cdot \mathbf{1}\!\left[\mathcal{B}_i^{p,(e)}(x) \neq \mathcal{B}_i^{p,(e)}(x')\right]
    }{
        \displaystyle\sum_{x' \in \mathcal{N}_K^{\mathrm{mut},(e)}(x)}
        \exp\!\bigl(-d(\mathcal{B}_i^{\phi,(e)}(x),\, \mathcal{B}_i^{\phi,(e)}(x'))\bigr)
    },
\end{equation}
yielding the raw reliability score $r_i^{(e)}(x) = \exp(-\hat{z}_i^{(e)}(x))$.

\paragraph{EMA Smoothing with Cosine Momentum Schedule.}
To stabilize reliability estimates across epochs, we apply an exponential moving average:
\begin{equation}
    \bar{r}_i^{(e)}(x) = m^{(e)} \cdot \bar{r}_i^{(e-1)}(x) + \bigl(1 - m^{(e)}\bigr) \cdot r_i^{(e)}(x),
\end{equation}
where the momentum $m^{(e)}$ follows a cosine schedule that increases monotonically:
\begin{equation}\label{eq:cosine_momentum}
    m^{(e)} = m_{\mathrm{start}} + (m_{\mathrm{end}} - m_{\mathrm{start}}) \cdot \frac{1 - \cos\!\left(\pi \cdot \frac{e}{E-1}\right)}{2}.
\end{equation}
This assigns low momentum in early epochs to allow rapid adaptation, and high momentum in later epochs to suppress noise once training stabilizes.
At initialization, we set $\bar{r}_i^{(0)} = r_i^{(0)}$ directly from the warmup pass.

\paragraph{Epoch-Indexed Training Objective.}
Substituting $\bar{r}_i^{(e-1)}$ into the gating functions of Eq.~\eqref{eq:gating}, the full epoch-indexed objective becomes:
\begin{align}\label{eq:training_objective}
    \mathcal{L}_{t}^{(e)} &= \mathbb{E}_{x}\!\left[
        T_{v \to t}^{(e-1)} \cdot D_{\mathrm{KL}}\!\left(\mathrm{sg}(p_v^{(e)}) \,\|\, p_t^{(e)}\right)
        + T_{t \to v}^{(e-1)} \cdot D_{\mathrm{KL}}\!\left(p_t^{(e-1)} \,\|\, p_t^{(e)}\right)
    \right] + \mathcal{R}(p_t^{(e)}), \\
    \mathcal{L}_{v}^{(e)} &= \mathbb{E}_{x}\!\left[
        T_{t \to v}^{(e-1)} \cdot D_{\mathrm{KL}}\!\left(\mathrm{sg}(p_t^{(e)}) \,\|\, p_v^{(e)}\right)
        + T_{v \to t}^{(e-1)} \cdot D_{\mathrm{KL}}\!\left(p_v^{(e-1)} \,\|\, p_v^{(e)}\right)
    \right] + \mathcal{R}(p_v^{(e)}),
\end{align}
where $\mathcal{R}(p) = \lambda_{\mathrm{IIC}}\mathcal{L}_{\mathrm{IIC}} - \lambda_{\mathrm{div}} H(\bar{p})$ denotes the batch-level regularization term.
Crucially, the gating weights $T^{(e-1)}$ are derived from the \emph{previous} epoch's reliability scores, preventing a circular dependency between the current supervision signal and the reliability estimate used to weight it. The complete procedure is summarized in Algorithm~\ref{alg:training}.
\begin{algorithm}[t]
\caption{Training of \proposed}
\label{alg:training}
\textbf{Input}: Unlabeled target data $\{x_j\}_{j=1}^{n}$, pretrained student $\theta_t$, frozen CLIP encoder $\phi_v$, learnable prompt $\theta_v$, total epochs $E$.\\
\textbf{Procedure}:
\begin{algorithmic}[1]
    \State \textbf{Initialisation}: Set $\theta_t \leftarrow \theta_s$; set prompt context $\boldsymbol{v}$ = ``a photo of a''.
    \State \textbf{Warmup}: Forward all $\{x_j\}$ through $\theta_t$ and $\phi_v$; populate banks $\mathcal{B}_i^{\phi}$, $\mathcal{B}_i^{p}$ for $i \in \{t, v\}$.
    \State Compute $r_i^{(0)}$ via Eq.~\eqref{eq:cut_estimate}; set $\bar{r}_i^{(0)} \leftarrow r_i^{(0)}$ for $i \in \{t, v\}$.
    \For{$e = 1, 2, \ldots, E$}
        \State Compute gating weights $T_{v \to t}^{(e-1)}$, $T_{t \to v}^{(e-1)}$ via Eq.~\eqref{eq:gating} using $\bar{r}_i^{(e-1)}$.
        \State \textbf{For each} mini-batch $\mathcal{B}$: forward through $\theta_t$ and $\phi_v$ to obtain $p_t^{(e)}$, $p_v^{(e)}$.
        \State \hspace{4.1em} Update banks: $\mathcal{B}_t^{\phi}[j] \leftarrow \phi_t^{(e)}(x_j)$, $\mathcal{B}_i^{p}[j] \leftarrow p_i^{(e)}(x_j)$ for $j \in \mathcal{B}$.
        \State \hspace{4.1em} Compute $\mathcal{L}_{t}^{(e)}$ and $\mathcal{L}_{v}^{(e)}$ via Eq.~\eqref{eq:training_objective}; update $\theta_t$ and $\boldsymbol{v}$.
        \State Compute $r_i^{(e)}$ from cached banks via Eq.~\eqref{eq:cut_estimate} (no extra forward pass).
        \State Update EMA: $\bar{r}_i^{(e)} \leftarrow m^{(e)} \cdot \bar{r}_i^{(e-1)} + (1 - m^{(e)}) \cdot r_i^{(e)}$ for $i \in \{t, v\}$.
    \EndFor
    \State \textbf{return} Adapted student model $\theta_t$.
\end{algorithmic}
\end{algorithm}

\section{Additional Results}
\label{app:additional_results}

\subsection{Comparison Against Paper-Reported Best Results}
\label{appendix:paper_best}
Tables~\ref{tab:office31_best}--\ref{tab:visda_best} compare \proposed against the 
best results reported in each paper, rather than our reproduced numbers. Methods for 
which per-task results were not reported in the original paper are excluded from the 
respective tables. \proposed achieves state-of-the-art performance on Office-31 
(+0.6\% over ProDe-V), Office-Home (+2.2\% over ProDe-V), and DomainNet-126 
(+0.8\% over ProDe-V), winning the majority of per-task comparisons on each 
benchmark. On VisDA-C, \proposed remains competitive (90.9\% vs.\ 91.0\% for 
ProDe-V), with per-class results detailed in Table~\ref{tab:visda_best}. These 
results confirm that the improvements reported in the main text are not artifacts 
of favorable reproduction conditions, but reflect genuine gains over the published 
state of the art.

\begin{table*}[h]
    \centering
    \scriptsize
    \caption{Comparison against paper-reported best results on \textbf{Office-31} (Acc, \%).}
    \label{tab:office31_best}
    \setlength{\tabcolsep}{1.5pt}
    \renewcommand{\arraystretch}{1.2}
    \begin{tabularx}{\textwidth}{lc | *{6}{C} | C}
        \toprule
        Method & Venue & \tiny{A$\to$D} & \tiny{A$\to$W} & \tiny{D$\to$A} & \tiny{D$\to$W} & \tiny{W$\to$A} & \tiny{W$\to$D} & Avg. \\
        \midrule
        Source   & --     & 79.1 & 76.5 & 59.9 & 95.5 & 61.4 & 98.8 & 78.5 \\
        \midrule
        SHOT     & ICML20 & 94.0 & 90.1 & 74.7 & 98.4 & 74.3 & 99.9 & 88.6 \\
        NRC      & NIPS21 & 96.0 & 90.8 & 75.3 & \textbf{\textcolor{red}{99.0}} & 75.0 & \textbf{\textcolor{red}{100.0}} & 89.4 \\
        AaD      & NIPS22 & 96.4 & 92.1 & 75.0 & \textbf{\textcolor{red}{99.1}} & 76.5 & \textbf{\textcolor{red}{100.0}} & 89.9 \\
        CoWA     & ICML22 & 94.4 & 95.2 & 76.2 & 98.5 & 77.6 & 99.8 & 90.3 \\
        ELR      & ICLR23 & 93.8 & 93.3 & 76.2 & 98.0 & 76.9 & \textbf{\textcolor{red}{100.0}} & 89.6 \\
        DIFO-V   & CVPR24 & 97.2 & 95.5 & 83.0 & 97.2 & 83.2 & 98.8 & 92.5 \\
        ProDe-V  & ICLR25 & 96.8 & \textbf{\textcolor{red}{96.4}} & 83.1 & 97.0 & 82.5 & 99.8 & 92.6 \\
        \midrule
        \rowcolor{gray!35}
        Ours     & NIPS26     & \textbf{\textcolor{red}{98.6}} & 95.2 & \textbf{\textcolor{red}{83.9}} & 98.2 & \textbf{\textcolor{red}{83.3}} & \textbf{\textcolor{red}{100.0}} & \textbf{\textcolor{red}{93.2}} \\
        \bottomrule
    \end{tabularx}
\end{table*}

\begin{table*}[h]
    \centering
    \scriptsize
    \caption{Comparison against paper-reported best results on \textbf{Office-Home} (Acc, \%).}
    \label{tab:officehome_best}
    \setlength{\tabcolsep}{1.5pt}
    \renewcommand{\arraystretch}{1.2}
    \begin{tabularx}{\textwidth}{lc | *{12}{C} | C}
        \toprule
        Method & Venue & \tiny{Ar$\to$Cl} & \tiny{Ar$\to$Pr} & \tiny{Ar$\to$Rw} & \tiny{Cl$\to$Ar} & \tiny{Cl$\to$Pr} & \tiny{Cl$\to$Rw} & \tiny{Pr$\to$Ar} & \tiny{Pr$\to$Cl} & \tiny{Pr$\to$Rw} & \tiny{Rw$\to$Ar} & \tiny{Rw$\to$Cl} & \tiny{Rw$\to$Pr} & Avg. \\
        \midrule
        Source   & --     & 43.6 & 67.0 & 73.9 & 49.9 & 60.2 & 62.6 & 51.7 & 40.9 & 72.6 & 64.1 & 46.3 & 78.1 & 59.2 \\
        \midrule
        SHOT     & ICML20 & 57.1 & 78.1 & 81.5 & 68.0 & 78.2 & 78.1 & 67.4 & 54.9 & 82.2 & 73.3 & 58.8 & 84.3 & 71.8 \\
        NRC      & NIPS21 & 57.7 & 80.3 & 82.0 & 68.1 & 79.8 & 78.6 & 65.3 & 56.4 & 83.0 & 71.0 & 58.6 & 85.6 & 72.2 \\
        AaD      & NIPS22 & 59.3 & 79.3 & 82.1 & 68.9 & 79.8 & 79.5 & 67.2 & 57.4 & 83.1 & 72.1 & 58.5 & 85.4 & 72.7 \\
        CoWA     & ICML22 & 56.9 & 78.4 & 81.0 & 69.1 & 80.0 & 79.9 & 67.7 & 57.2 & 82.4 & 72.8 & 60.5 & 84.5 & 72.5 \\
        ELR      & ICLR23 & 58.4 & 78.7 & 81.5 & 69.2 & 79.5 & 79.3 & 66.3 & 58.0 & 82.6 & 73.4 & 59.8 & 85.1 & 72.6 \\
        DIFO-V   & CVPR24 & 70.6 & 90.6 & 88.8 & 82.5 & 90.6 & 88.8 & 80.9 & 70.1 & 88.9 & 83.4 & 70.5 & 91.2 & 83.1 \\
        ProDe-V  & ICLR25 & 72.7 & 92.3 & 90.5 & 82.5 & 91.5 & 90.7 & 82.5 & 72.5 & 90.8 & 83.0 & 72.6 & 92.2 & 84.5 \\
        \midrule
        \rowcolor{gray!35}
        Ours     & NIPS26     & \textbf{\textcolor{red}{76.9}} & \textbf{\textcolor{red}{93.0}} & \textbf{\textcolor{red}{91.9}} & \textbf{\textcolor{red}{85.9}} & \textbf{\textcolor{red}{92.5}} & \textbf{\textcolor{red}{91.6}} & \textbf{\textcolor{red}{85.0}} & \textbf{\textcolor{red}{76.7}} & \textbf{\textcolor{red}{92.0}} & \textbf{\textcolor{red}{85.7}} & \textbf{\textcolor{red}{76.1}} & \textbf{\textcolor{red}{93.2}} & \textbf{\textcolor{red}{86.7}} \\
        \bottomrule
    \end{tabularx}
\end{table*}

\begin{table*}[h]
    \centering
    \scriptsize
    \caption{Comparison against paper-reported best results on \textbf{DomainNet-126} (Acc, \%).}
    \label{tab:domainnet_best}
    \setlength{\tabcolsep}{1.5pt}
    \renewcommand{\arraystretch}{1.2}
    \begin{tabularx}{\textwidth}{lc | *{12}{C} | C}
        \toprule
        Method & Venue & \tiny{C$\to$P} & \tiny{C$\to$R} & \tiny{C$\to$S} & \tiny{P$\to$C} & \tiny{P$\to$R} & \tiny{P$\to$S} & \tiny{R$\to$C} & \tiny{R$\to$P} & \tiny{R$\to$S} & \tiny{S$\to$C} & \tiny{S$\to$P} & \tiny{S$\to$R} & Avg. \\
        \midrule
        Source   & --     & 44.7 & 59.9 & 47.5 & 53.3 & 75.3 & 46.2 & 55.1 & 62.6 & 46.4 & 55.2 & 50.8 & 59.6 & 54.7 \\
        \midrule
        ELR      & ICLR23 & 77.2 & 91.2 & 68.5 & 79.8 & 92.7 & 67.5 & 75.6 & 82.2 & 65.7 & 85.1 & 77.6 & 89.3 & 79.4 \\
        DIFO-V   & CVPR24 & 76.6 & 87.2 & 74.9 & 80.0 & 87.4 & 75.6 & 80.8 & 77.3 & 75.5 & 80.5 & 76.7 & 87.3 & 80.0 \\
        ProDe-V  & ICLR25 & \textbf{\textcolor{red}{83.2}} & \textbf{\textcolor{red}{92.4}} & 79.0 & 85.0 & \textbf{\textcolor{red}{92.3}} & 79.3 & 85.5 & \textbf{\textcolor{red}{83.1}} & 79.1 & 85.5 & \textbf{\textcolor{red}{83.4}} & \textbf{\textcolor{red}{92.4}} & 85.0 \\
        \midrule
        \rowcolor{gray!35}
        Ours     & NIPS26     & 82.8 & 92.3 & \textbf{\textcolor{red}{81.8}} & \textbf{\textcolor{red}{86.2}} & 92.2 & \textbf{\textcolor{red}{81.0}} & \textbf{\textcolor{red}{86.9}} & 82.9 & \textbf{\textcolor{red}{81.7}} & \textbf{\textcolor{red}{86.6}} & 83.0 & 92.3 & \textbf{\textcolor{red}{85.8}} \\
        \bottomrule
    \end{tabularx}
\end{table*}

\begin{table*}[h]
    \centering
    \scriptsize
    \caption{Comparison against paper-reported best results on \textbf{VisDA-C} (per-class Acc, \%).}
    \label{tab:visda_best}
    \setlength{\tabcolsep}{2pt}
    \renewcommand{\arraystretch}{1.2}
    \begin{tabularx}{\textwidth}{lc | *{12}{C} | C}
        \toprule
        Method & Venue & \tiny{plane} & \tiny{bcycl} & \tiny{bus} & \tiny{car} & \tiny{horse} & \tiny{knife} & \tiny{mcycl} & \tiny{person} & \tiny{plant} & \tiny{sktbrd} & \tiny{train} & \tiny{truck} & Avg. \\
        \midrule
        Source   & --     & 60.7 & 21.7 & 50.9 & 68.5 & 71.8 &  5.3 & 86.3 & 20.2 & 67.1 & 43.2 & 83.3 & 10.6 & 49.1 \\
        \midrule
        SHOT     & ICML20 & 94.3 & 88.5 & 80.1 & 57.3 & 93.1 & 94.9 & 80.7 & 80.3 & 91.5 & 89.1 & 86.3 & 58.2 & 82.9 \\
        NRC      & NIPS21 & 96.8 & 91.3 & 82.4 & 62.4 & 96.2 & 95.9 & 86.1 & 80.6 & 94.8 & 94.1 & 90.4 & 59.7 & 85.9 \\
        AaD      & NIPS22 & 97.4 & 90.5 & 80.8 & 76.2 & 97.3 & 96.1 & 89.8 & 82.9 & 95.5 & 93.0 & 92.0 & 64.7 & 88.0 \\
        AdaCon   & CVPR22 & 97.0 & 84.7 & 84.0 & 77.3 & 96.7 & 93.8 & 91.9 & 84.8 & 94.3 & 93.1 & 94.1 & 49.7 & 86.8 \\
        CoWA     & ICML22 & 96.2 & 89.7 & 83.9 & 73.8 & 96.4 & 97.4 & 89.3 & 86.8 & 94.6 & 92.1 & 88.7 & 53.8 & 86.9 \\
        ELR      & ICLR23 & 97.1 & 89.7 & 82.7 & 62.0 & 96.2 & 97.0 & 87.6 & 81.2 & 93.7 & 94.1 & 90.2 & 58.6 & 85.8 \\
        PLUE     & CVPR23 & 97.3 & \textbf{\textcolor{red}{96.2}} & \textbf{\textcolor{red}{90.5}} & \textbf{\textcolor{red}{91.8}} & 90.0 & 94.2 & 87.4 & \textbf{\textcolor{red}{87.7}} & \textbf{\textcolor{red}{97.0}} & 84.3 & 93.0 & \textbf{\textcolor{red}{81.0}} & 90.0 \\
        DIFO-V   & CVPR24 & 97.5 & 89.0 & 90.8 & 83.5 & 97.8 & 97.3 & 93.2 & 83.5 & 95.2 & \textbf{\textcolor{red}{96.8}} & 93.7 & 65.9 & 90.3 \\
        ProDe-V  & ICLR25 & 98.5 & 92.4 & 86.6 & 80.5 & \textbf{\textcolor{red}{98.1}} & 98.0 & 92.3 & 84.3 & 94.7 & 97.0 & 94.1 & 75.6 & \textbf{\textcolor{red}{91.0}} \\
        \midrule
        \rowcolor{gray!35}
        Ours     & NIPS26     & \textbf{\textcolor{red}{98.7}} & 91.9 & 86.8 & 82.3 & 97.6 & \textbf{\textcolor{red}{98.1}} & \textbf{\textcolor{red}{94.1}} & 86.8 & 94.6 & 94.8 & \textbf{\textcolor{red}{94.5}} & 70.9 & 90.9 \\
        \bottomrule
    \end{tabularx}
\end{table*}

\clearpage
\subsection{VisDA Full Results}
\label{app:visda_full}

\begin{table*}[h]
    \centering
    \scriptsize
    \caption{Full per-class results on VisDA-C (Synthesis $\rightarrow$ Real, acc, \%).}
    \label{tab:visda_full}
    \setlength{\tabcolsep}{3pt}
    \renewcommand{\arraystretch}{1.2}
    \begin{tabularx}{\textwidth}{lc | *{12}{C} | C}
        \toprule
        Method & Venue & plane & bcycl & bus & car & horse & knife & mcycl & person & plant & sktbrd & train & truck & Per-class \\
        \midrule
        Source   & --     & 60.7 & 21.7 & 50.9 & 68.5 & 71.8 & 5.3  & 86.3 & 20.2 & 67.1 & 43.2 & 83.3 & 10.6 & 49.1 \\
        \midrule
        SHOT     & ICML20 & 94.9 & 86.9 & 79.7 & 54.8 & 93.6 & 94.1 & 78.5 & 81.4 & 90.6 & 90.0 & 84.5 & 58.3 & 82.3 \\
        NRC      & NIPS21 & 97.2 & 91.5 & 83.5 & 39.3 & 96.4 & 96.2 & 82.8 & 79.7 & 93.8 & 93.3 & 91.2 & 55.8 & 83.4 \\
        AaD      & NIPS22 & 97.5 & 91.8 & 70.6 & 45.5 & 92.3 & 90.0 & 65.6 & 65.6 & 95.0 & 88.4 & 87.5 & 56.0 & 78.8 \\
        AdaCon   & CVPR22 & 88.4 & 69.1 & 78.5 & 55.2 & 85.4 & 91.6 & 79.7 & 55.8 & 80.3 & 48.4 & 64.7 & 8.9  & 67.2 \\
        CoWA     & ICML22 & 95.5 & 84.2 & 83.0 & 68.9 & 94.3 & 96.9 & 88.8 & 83.1 & 91.3 & 89.0 & 85.5 & 51.5 & 84.3 \\
        ELR      & ICLR23 & 97.2 & 89.7 & 84.0 & 60.1 & 95.9 & 96.7 & 87.1 & 81.1 & 93.3 & 92.9 & 91.0 & 57.3 & 85.5 \\
        PLUE     & CVPR23 & 97.1 & 88.5 & 84.7 & 59.0 & 96.1 & 94.5 & 89.2 & \textbf{\textcolor{red}{86.6}} & 93.7 & 86.9 & 92.3 & 70.5 & 86.6 \\
        DIFO-V   & CVPR24 & 97.7 & 87.7 & 86.4 & 74.7 & 97.3 & \textbf{\textcolor{red}{98.1}} & 92.9 & 85.5 & 92.2 & \textbf{\textcolor{red}{96.2}} & 93.0 & 75.7 & 89.8 \\
        ProDe-V  & ICLR25 & 98.5 & \textbf{\textcolor{red}{93.4}} & 83.9 & 78.3 & \textbf{\textcolor{red}{98.1}} & 97.8 & 90.6 & 84.3 & 94.6 & 96.4 & 93.9 & \textbf{\textcolor{red}{77.1}} & 90.6 \\
        \midrule
        \rowcolor{gray!35}
        Ours     & NIPS26     & \textbf{\textcolor{red}{98.6}} & 92.1 & \textbf{\textcolor{red}{87.3}} & \textbf{\textcolor{red}{83.3}} & 97.8 & \textbf{\textcolor{red}{98.1}} & \textbf{\textcolor{red}{94.1}} & 86.0 & \textbf{\textcolor{red}{94.8}} & 94.6 & \textbf{\textcolor{red}{94.5}} & 69.5 & \textbf{\textcolor{red}{90.9}} \\
        \bottomrule
    \end{tabularx}
\end{table*}
Table~\ref{tab:visda_full} presents the full per-class breakdown on VisDA-C. 
\proposed achieves the highest per-class average (90.9\%), outperforming 
ProDe-V by 0.3\%p despite falling behind on individual categories such as 
\textit{bcycl}, \textit{horse}, and \textit{truck}. The largest gap occurs 
on \textit{truck} ($-7.6$\%p vs.\ ProDe-V), which we attribute to the 
visually ambiguous nature of the category under domain shift. Notably, 
\proposed achieves strong gains on \textit{car} (+5.0\%p) and \textit{bus} 
(+3.4\%p) over ProDe-V, categories that are particularly sensitive to 
appearance-level domain shift, suggesting that reliability-gated mutual 
correction is especially effective where visual statistics differ 
significantly between domains.

\subsection{Computational Analysis}
\label{app:compute}

\begin{wraptable}{r}{0.48\textwidth}
\vspace{-1.2em}
\scriptsize
\centering
\caption{\textbf{Computational cost analysis between \proposed~(ours) and ProDe.}}
\label{tab:efficiency}
\setlength{\tabcolsep}{3pt}
\begin{tabular}{lrrrrr}
\toprule
 & \multicolumn{4}{c}{\textbf{\proposed} (Ours)} & \multicolumn{1}{c}{\textbf{ProDe}} \\
\cmidrule(lr){2-5}\cmidrule(lr){6-6}
Dataset
  & \shortstack{Avg.\\KNN}
  & \shortstack{Avg.\\Train}
  & \shortstack{Avg.\\Infer.}
  & \shortstack{Total}
  & \shortstack{Total} \\
\midrule
Office-31      & 0.05\,s &   8.92\,s &  6.57\,s & \textbf{0.78\,h}  &  1.11\,h \\
Office-Home    & $<$0.01\,s & 17.85\,s & 15.83\,s & \textbf{5.62\,h}  &  5.87\,h \\
DomainNet-126  & 0.25\,s &  72.21\,s & 44.86\,s & \textbf{19.55\,h} & 29.58\,h \\
VisDA          & 0.40\,s & 160.80\,s & 80.70\,s & \textbf{2.02\,h}  &  3.49\,h \\
\bottomrule
\end{tabular}
\vspace{-0.5em}
\end{wraptable}
Table~\ref{tab:efficiency} reports the per-component runtime of \proposed alongside the total wall-clock time of ProDe.\footnote{All experiments are conducted on a single NVIDIA RTX PRO 6000 Blackwell Server Edition GPU (97\,GB VRAM).} \proposed maintains two models and computes per-sample reliability estimates via a mutual $k$-NN graph at each epoch, which introduces additional cost over a single-model baseline. However, two design choices keep this overhead manageable. First, the cut statistic is computed \emph{once per epoch} from feature representations cached during the preceding optimization forward pass, avoiding any separate full forward pass over the target domain; the resulting per-epoch KNN cost remains below 0.5\,s across all benchmarks. Second, target-domain visual features for the CLIP peer are extracted and cached prior to adaptation and reused throughout training via an EMA-updated reliability bank, eliminating repeated invocations of the CLIP backbone during training. As a result, despite the additional complexity of bidirectional mutual correction, \proposed completes adaptation faster than ProDe across all four benchmarks.

\begin{figure}[h]
    \centering
    \includegraphics[width=1.0\linewidth]{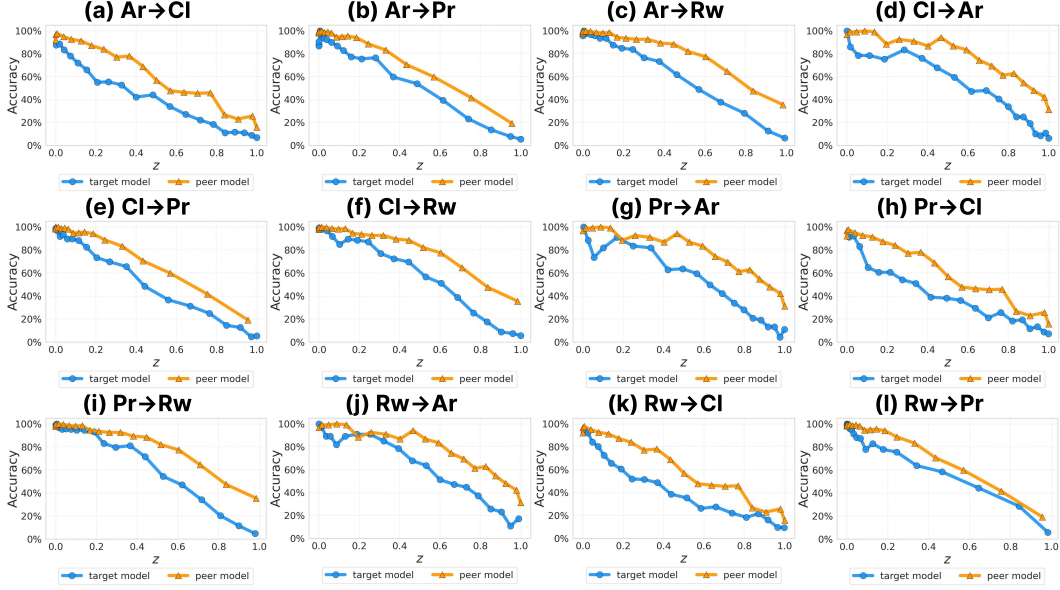}
    \caption{\textbf{Accuracy vs.\ cut statistic $z$ for all 12 Office-Home tasks (CLIP peer).}}
    \label{fig:oh_observe}
\end{figure}

\subsection{Extended Empirical Validation of Cut Statistic Properties}
\label{appendix:extended_obs}
\begin{figure}[t]
    \centering
    \includegraphics[width=1.0\linewidth]{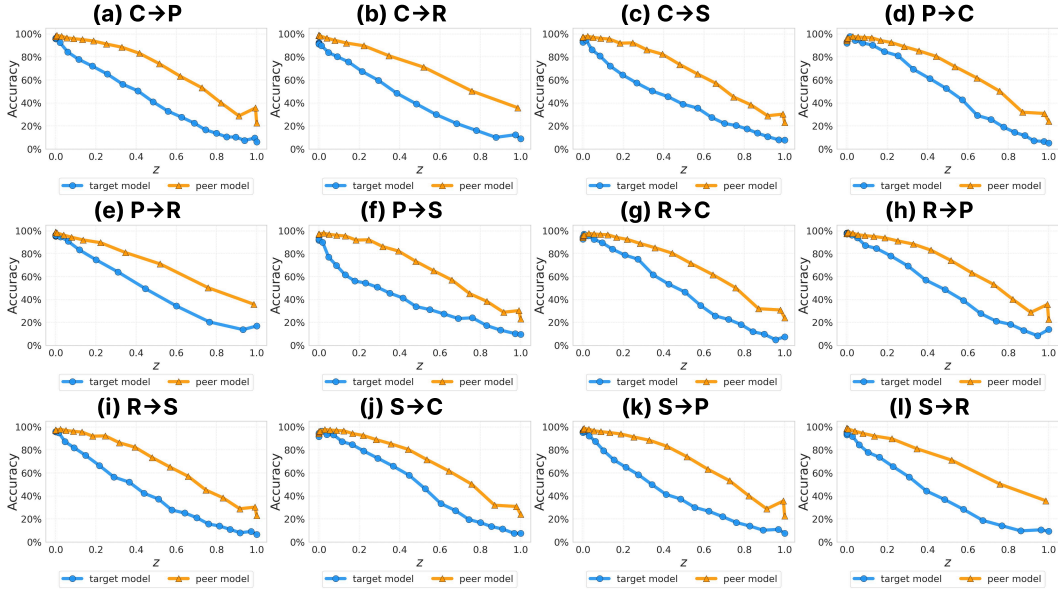}
    \caption{\textbf{Accuracy vs.\ cut statistic $z$ for all 12 DomainNet-126 tasks (CLIP peer).}}
    \label{fig:dn_observe}
\end{figure}
\begin{figure}[t]
    \centering
    \includegraphics[width=1.0\linewidth]{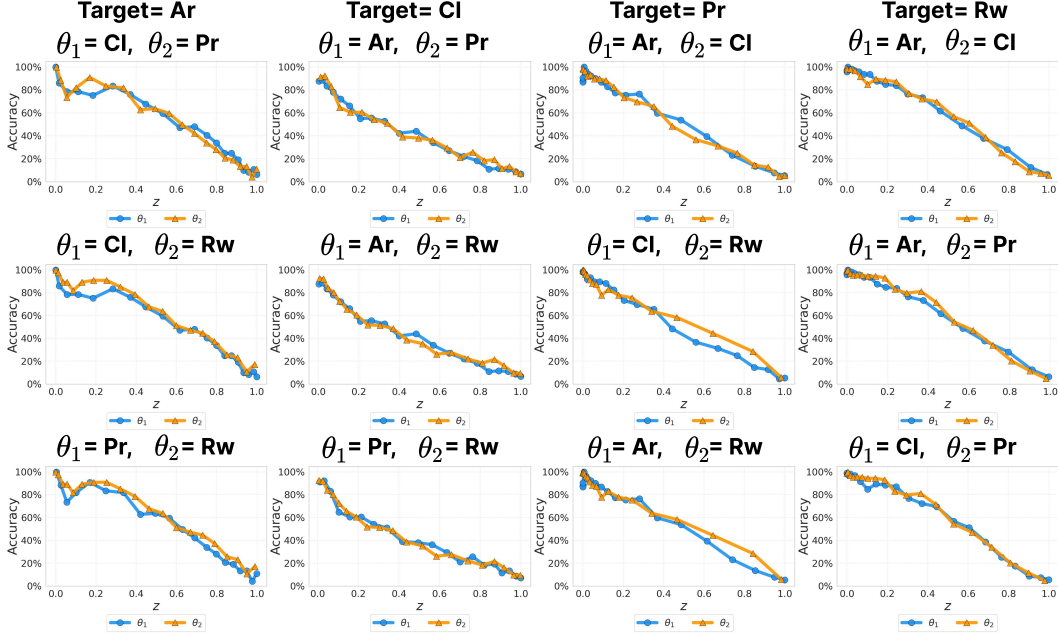}
    \caption{\textbf{Accuracy vs.\ cut statistic $z$ for all 12 Office-Home source-pair combinations (dual-ResNet peer).}}
    \label{fig:hetero_observe}
\end{figure}
We provide extended empirical validation of the two key properties underlying our 
framework: (i) approximate monotonicity of predictive accuracy with respect to the 
cut statistic (Observation~\ref{obs:monotonicity}), and (ii) spatial separation of 
complementary correction regions in the joint cut plane. We examine these properties 
across three settings: Office-Home with a CLIP peer 
(Figures~\ref{fig:oh_observe},~\ref{fig:oh_density}), DomainNet-126 with a CLIP peer 
(Figures~\ref{fig:dn_observe},~\ref{fig:dn_density}), and Office-Home with a 
heterogeneous dual-ResNet peer (Figures~\ref{fig:hetero_observe},~\ref{fig:hetero_density}).

\paragraph{Approximate Monotonicity.}
Across all three settings and all 12 adaptation tasks per setting, both models 
consistently exhibit approximately monotonically decreasing accuracy with respect to 
their own cut statistic $z$, confirming that Observation~\ref{obs:monotonicity} holds 
broadly regardless of benchmark, peer architecture, and inter-model accuracy gap. 
This robustness is most directly evidenced by contrasting the two peer configurations: 
in the CLIP-based settings (Figures~\ref{fig:oh_observe},~\ref{fig:dn_observe}), 
$\theta_v$ maintains a large and persistent accuracy advantage over $\theta_t$ across 
the full range of $z$, whereas in the heterogeneous dual-ResNet setting 
(Figure~\ref{fig:hetero_observe}), the curves of $\theta_1$ and $\theta_2$ run in close 
proximity and frequently cross at intermediate cut values. Despite these contrasting 
regimes -- one with strong directional asymmetry and one with near-symmetric accuracy -- 
the monotonic relationship between cut value and predictive accuracy is preserved in 
both cases, confirming that the cut statistic functions as a reliable per-model 
reliability proxy independent of the degree of inter-model asymmetry.

\paragraph{Spatial Separation of Correction Regions.}
Figures~\ref{fig:oh_density},~\ref{fig:dn_density}, and~\ref{fig:hetero_density} 
display KDE density plots of $\mathcal{S}_{c\bar{c}}$ (blue) and $\mathcal{S}_{\bar{c}c}$ 
(orange) in the joint cut plane. The red diagonal line marks $\Delta(x) = 0$, i.e., 
equal cut values for both models; the baseline thresholds $\gamma_0$ and $\xi_0$ from 
Proposition~\ref{prop:threshold} reside in the upper and lower half-spaces relative 
to this line, respectively, and need not coincide with any fixed offset.

Across all three settings, the blue mode concentrates in the upper half-space 
($\Delta(x) > 0$) and the orange mode in the lower half-space ($\Delta(x) < 0$), 
confirming that the two complementary correction regions are separated by the 
$\Delta(x) = 0$ boundary as predicted. The degree of this separation, however, varies 
with the choice of target domain. In most tasks, the two modes are pushed toward 
opposite extremes of the joint cut plane, yielding a wide effective margin and making 
the identification of $\gamma_0$ and $\xi_0$ straightforward. When \textbf{Art (Ar)} 
serves as the target domain, both distributions shift toward the red line with smaller 
$|\Delta(x)|$, compressing the separation. Nevertheless, the relative ordering of the 
two modes -- blue above and orange below the $\Delta(x) = 0$ boundary -- is preserved 
even in this compressed regime, ensuring that valid baseline thresholds $\gamma_0 > 0$ 
and $\xi_0 > 0$ satisfying Proposition~\ref{prop:threshold} can still be identified 
within their respective half-spaces.

\begin{figure}[t]
    \centering
    \includegraphics[width=1.0\linewidth]{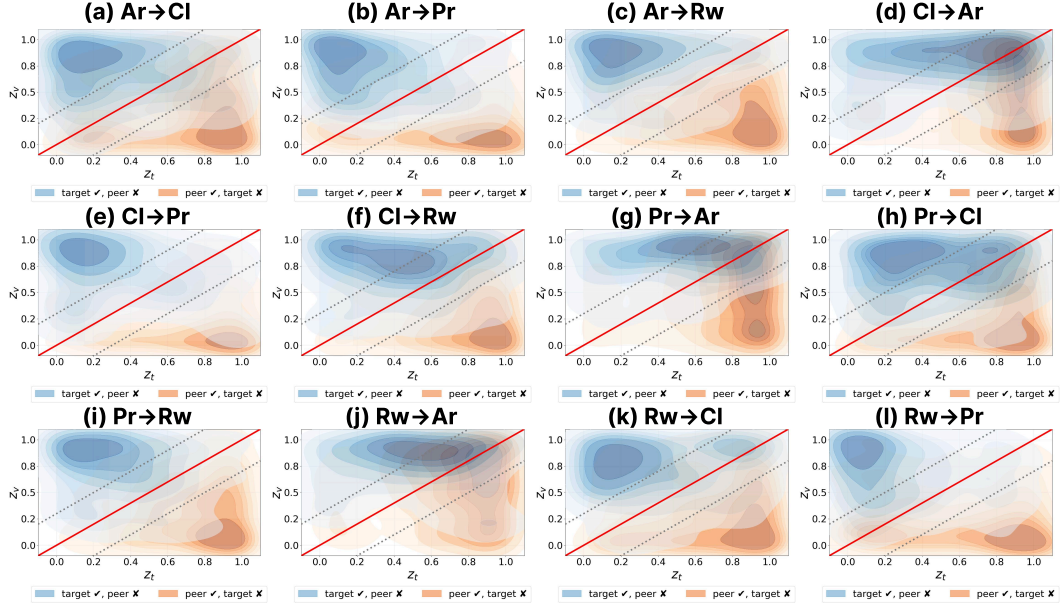}
    \caption{\textbf{Joint cut density of $\mathcal{S}_{c\bar{c}}$ and $\mathcal{S}_{\bar{c}c}$ for all 12 Office-Home tasks (CLIP peer).}}
    \label{fig:oh_density}
\end{figure}
\begin{figure}[t]
    \centering
    \includegraphics[width=1.0\linewidth]{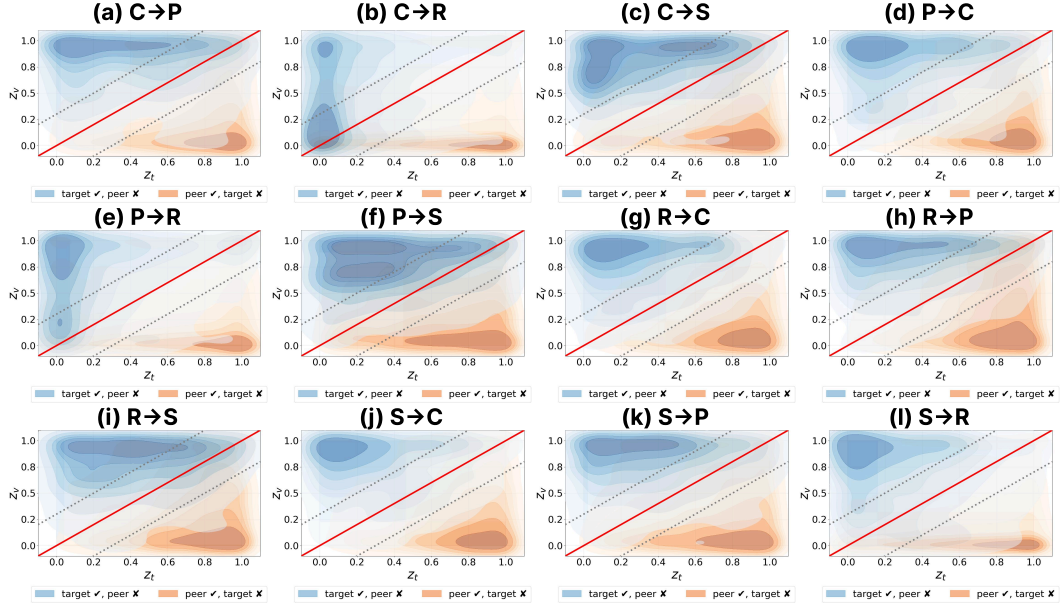}
    \caption{\textbf{Joint cut density of $\mathcal{S}_{c\bar{c}}$ and $\mathcal{S}_{\bar{c}c}$ for all 12 DomainNet-126 tasks (CLIP peer).}}
    \label{fig:dn_density}
\end{figure}
\begin{figure}[t]
    \centering
    \includegraphics[width=1.0\linewidth]{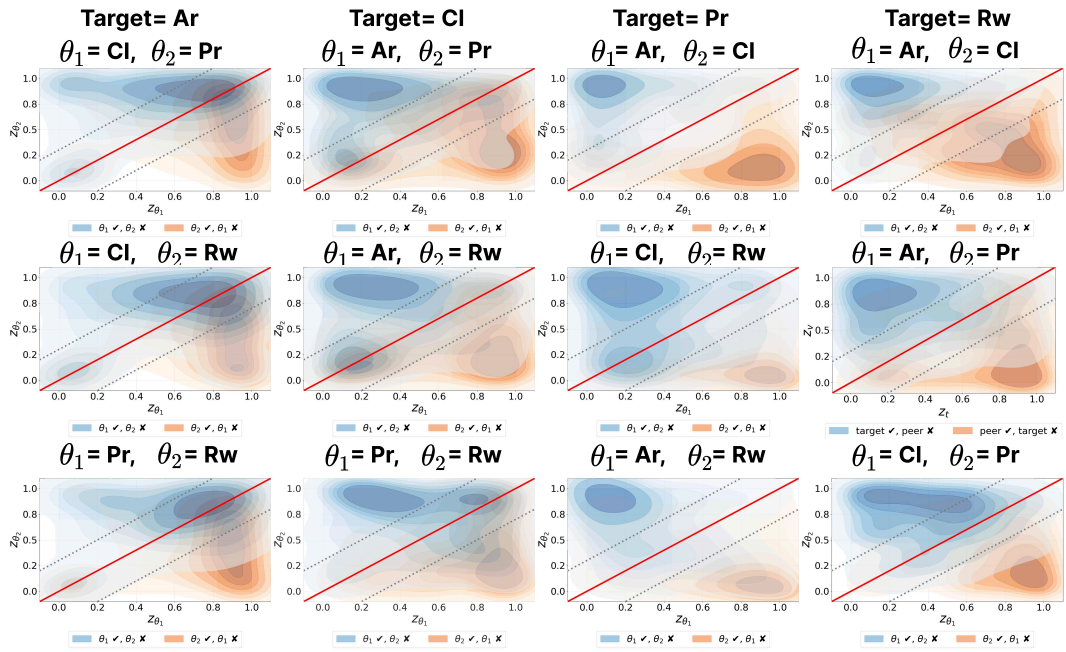}
    \caption{\textbf{Joint cut density of $\mathcal{S}_{c\bar{c}}$ and $\mathcal{S}_{\bar{c}c}$ for all 12 Office-Home source-pair combinations (dual-ResNet peer).}}
    \label{fig:hetero_density}
\end{figure}
\clearpage

\paragraph{Temporal Stability of Monotonicity.}
\begin{figure}
    \centering
    \includegraphics[width=1.0\linewidth]{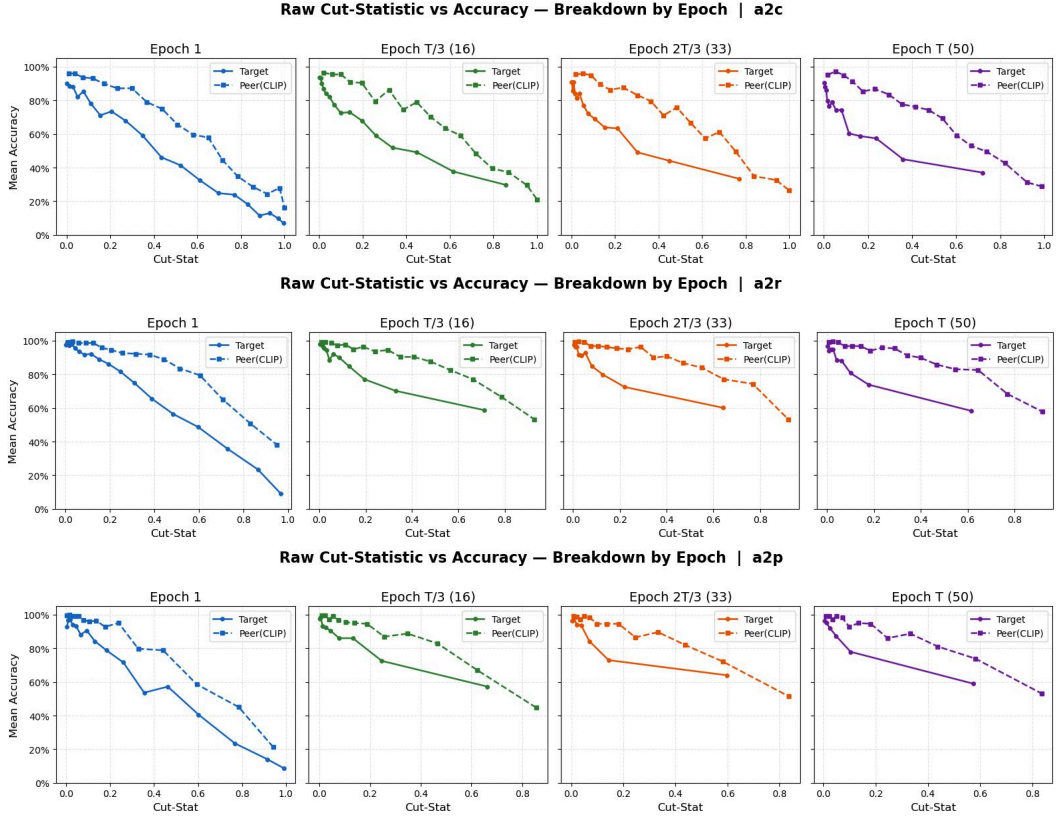}
    \caption{\textbf{Temporal stability of the cut–accuracy monotonicity.} Mean accuracy versus cut statistic $z$ for the target model (solid) and CLIP peer (dashed) at four checkpoints (Epoch~1, $T/3$, $2T/3$, $T$) on three Office-Home tasks. The monotonically decreasing trend and the relative ordering between the two curves are preserved throughout adaptation, while the target curve progressively converges toward the peer.}
    \label{fig:Epoch_Wise_Observation}
\end{figure}
Beyond the snapshot validation in Observation 3.1, a stronger property would be that the cut–accuracy monotonicity holds \emph{throughout} adaptation—ensuring that the reliability scores $\bar{r}_i$ accumulated through EMA remain well-calibrated as both models evolve, and that reliability-gated supervision retains its fidelity precisely when mutual correction is most active. Figure~\ref{fig:Epoch_Wise_Observation} verifies this stronger property by tracking the cut–accuracy relationship across four representative checkpoints (Epoch~1, $T/3$, $2T/3$, and $T$) on three Office-Home tasks (Ar$\to$Cl, Ar$\to$Rw, Ar$\to$Pr). Across every task and every checkpoint, both the target model and the CLIP peer maintain a clear monotonically decreasing trend, and the relative ordering of the two curves is consistently preserved. Even more encouragingly, as training progresses the target curve steadily shifts upward and tightens toward the peer curve, directly reflecting the joint improvement driven by reliability-gated mutual correction—while the qualitative cut–accuracy shape remains intact. These results establish the cut statistic as a temporally robust reliability proxy: its validity extends seamlessly from initialization through the entire training trajectory, providing a stable foundation on which reliability-aware supervision continues to deliver well-calibrated guidance throughout adaptation.

\clearpage
\subsection{Extended Correctness Pattern Transition Results}
\label{appendix:state_transition}

\begin{figure}[!t]
    \centering
    \includegraphics[width=1.0\linewidth]{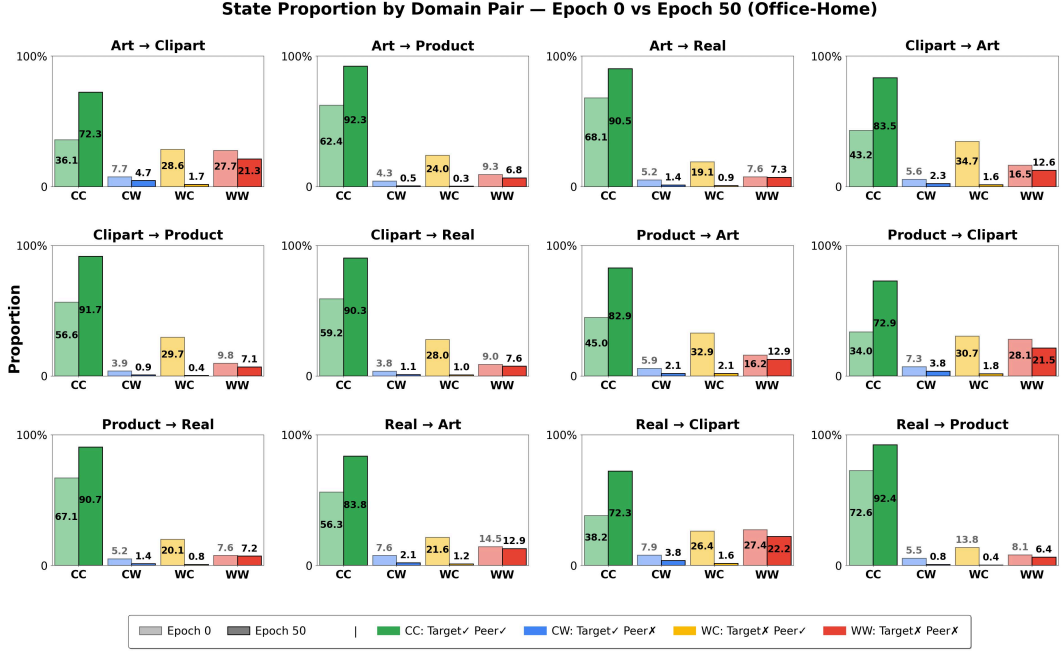}
    \caption{\textbf{Correctness pattern transition (Epoch 0 vs.\ 50) for all 12 Office-Home 
    tasks (CLIP peer).}}
    \label{fig:oh_transition}
\end{figure}
\begin{figure}[!t]
    \centering
    \includegraphics[width=1.0\linewidth]{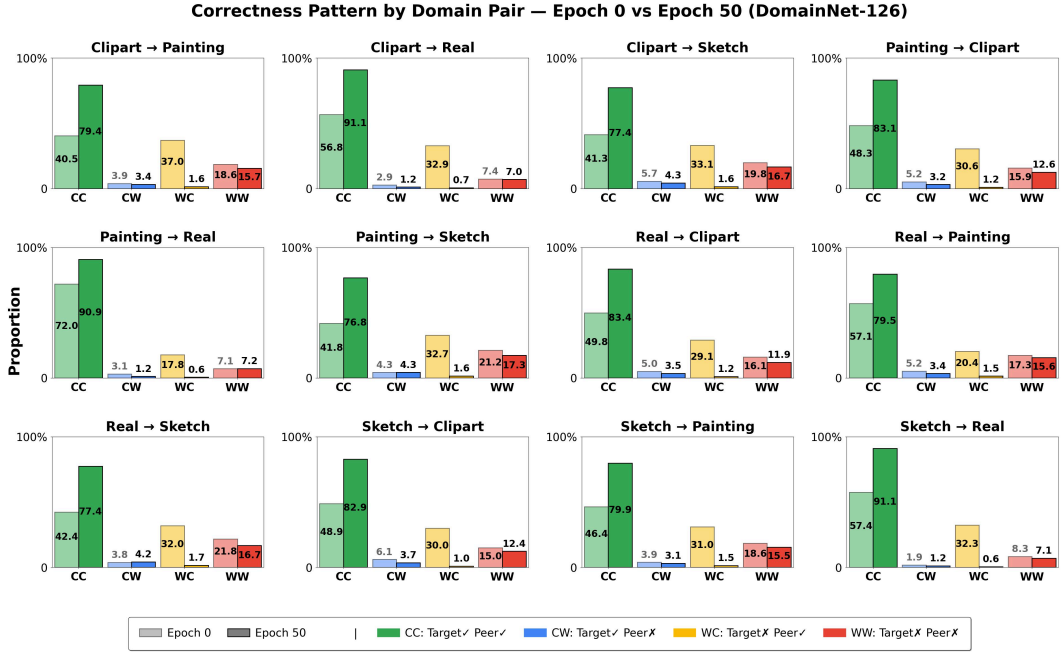}
    \caption{\textbf{Correctness pattern transition (Epoch 0 vs.\ 50) for all 12 DomainNet-126 
    tasks (CLIP peer).}}
    \label{fig:dn_transition}
\end{figure}
\begin{figure}[!t]
    \centering
    \includegraphics[width=1.0\linewidth]{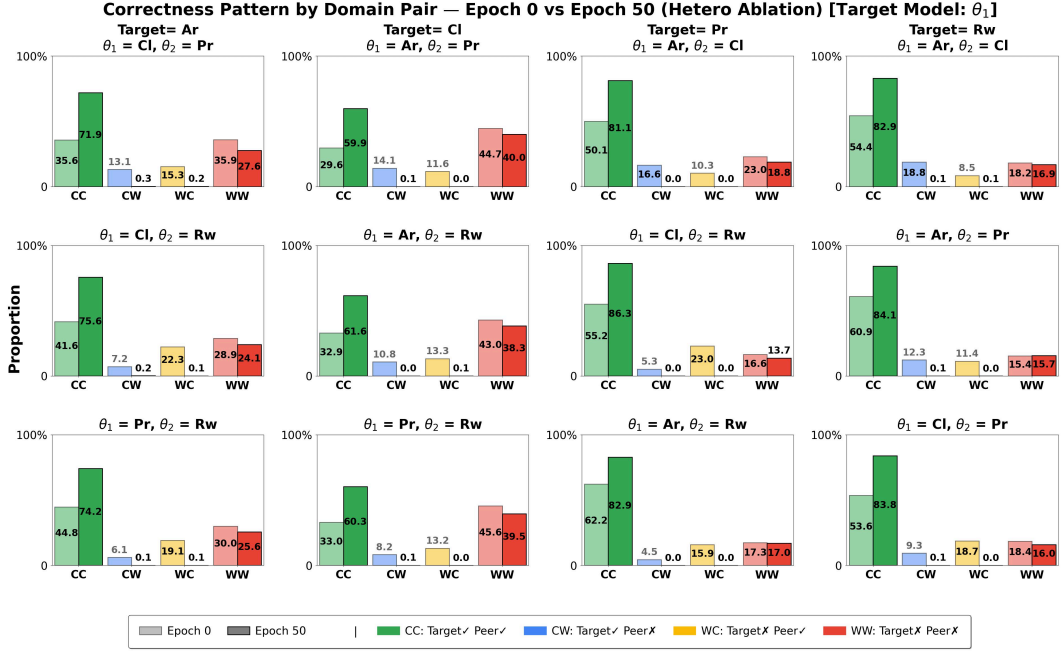}
    \caption{\textbf{Correctness pattern transition (Epoch 0 vs.\ 50) for all 12 Office-Home 
    source-pair combinations (dual-ResNet peer, $\theta_1$).}}
    \label{fig:hetero_transition}
\end{figure}

Figures~\ref{fig:oh_transition},~\ref{fig:dn_transition}, and~\ref{fig:hetero_transition} 
report the correctness pattern transitions from Epoch 0 to Epoch 50 across all 12 
adaptation tasks for Office-Home (CLIP peer), DomainNet-126 (CLIP peer), and Office-Home 
(dual-ResNet peer), respectively. Each bar group decomposes the target domain into the 
four oracle subsets -- $\mathcal{S}_{cc}$, $\mathcal{S}_{c\bar{c}}$, 
$\mathcal{S}_{\bar{c}c}$, and $\mathcal{S}_{\bar{c}\bar{c}}$ -- before and after 
adaptation.

\paragraph{Consistent Pattern Across All Settings.}
Three transitions are observed consistently across all benchmarks and peer configurations. 
First, the jointly correct proportion $\mathcal{S}_{cc}$ increases substantially in every 
task, confirming that bidirectional knowledge transfer effectively drives both models 
toward agreement on correct predictions. Second, the one-sided correct sets 
$\mathcal{S}_{c\bar{c}}$ and $\mathcal{S}_{\bar{c}c}$ nearly vanish after adaptation -- 
the populations that initially provided the source of mutual correction are progressively 
absorbed into $\mathcal{S}_{cc}$, indicating that cross-model supervision successfully 
resolves asymmetric errors in both directions. Third, the jointly incorrect set 
$\mathcal{S}_{\bar{c}\bar{c}}$ also exhibits a modest but consistent decrease across most 
tasks, suggesting that neighborhood-level generalization from corrected samples propagates 
improvement even to samples that neither model could directly supervise.

\paragraph{Amplified Effect in the Dual-ResNet Setting.}
While the above pattern holds in both peer configurations, two notable differences 
emerge in the dual-ResNet setting (Figure~\ref{fig:hetero_transition}). First, the 
one-sided sets $\mathcal{S}_{c\bar{c}}$ and $\mathcal{S}_{\bar{c}c}$ collapse more 
sharply toward zero compared to the CLIP-based settings -- in many source-pair 
combinations, both sets reduce to below 0.1\% after adaptation, a more complete 
elimination than observed in the CLIP case. This more thorough resolution of asymmetric 
errors reflects the structural symmetry of the two ResNet models: since neither model 
maintains a persistent directional advantage, correction signals flow bidirectionally 
with equal effectiveness, leaving minimal residual one-sided errors. Second, despite 
starting from lower initial $\mathcal{S}_{cc}$ proportions -- a consequence of both 
models being source-pretrained without the semantic priors of a vision-language model -- 
the absolute gain in $\mathcal{S}_{cc}$ is comparably large, demonstrating that 
\proposed can leverage complementary failure modes effectively even when the peer model 
is architecturally identical and initially less accurate than in the CLIP setting.

\subsection{Effect of Peer Complementarity}
\label{appendix:peer_complementarity}

SafeCut relies on complementary predictions between the two peer models: when their failure modes differ, one model can provide corrective supervision on samples where the other model is unreliable. To quantitatively examine how this complementarity affects adaptation, we use the 12 dual-ResNet settings reported in Table~\ref{tab:symmetric}. For each pair, we compute the prediction agreement before adaptation, defined as the proportion of target samples for which the two models predict the same class. We then measure the adaptation gain as the average accuracy improvement of the two models after \proposed adaptation.

\begin{table}[h]
    \centering
    \small
    \caption{Effect of peer complementarity in the dual-ResNet setting
    of Table~\ref{tab:symmetric}. Agreement denotes the prediction agreement between the
    two models before adaptation, and Gain denotes their average accuracy
    improvement after SafeCut adaptation.}
    \label{tab:peer_complementarity_full}
    \setlength{\tabcolsep}{5pt}
    \begin{tabular}{ccc|cc}
        \toprule
        Target & Source 1 & Source 2 & Agreement (\%) & Gain (\%p) \\
        \midrule
        Cl & Ar & Pr & 41.1 & 17.6 \\
        Ar & Cl & Pr & 43.3 & 21.3 \\
        Cl & Ar & Rw & 46.1 & 16.7 \\
        Cl & Pr & Rw & 46.4 & 16.8 \\
        Ar & Cl & Rw & 48.4 & 18.7 \\
        Ar & Pr & Rw & 53.6 & 16.5 \\
        Pr & Ar & Cl & 57.1 & 17.7 \\
        Rw & Cl & Pr & 58.8 & 16.3 \\
        Rw & Ar & Cl & 59.1 & 14.9 \\
        Pr & Cl & Rw & 60.6 & 17.2 \\
        Rw & Ar & Pr & 67.3 & 11.0 \\
        Pr & Ar & Rw & 69.6 & 10.4 \\
        \bottomrule
    \end{tabular}
\end{table}
We observe a strong negative relationship between pre-adaptation prediction agreement and the gain achieved by SafeCut (Pearson $r=-0.81$). Model pairs with lower agreement generally obtain larger improvements, supporting the premise that complementary predictions create more opportunities for cross-model correction. As the two models become more similar, the achievable benefit decreases, consistent with the fact that the one-sided correctness regions become less prevalent. Importantly, this reduction is gradual rather than harmful over the evaluated range. Even for the pair with the highest observed agreement (69.6\%), \proposed yields a +10.4 percentage-point improvement, and no performance degradation is observed across the 12 evaluated pairs. These results empirically support the role of peer complementarity in \proposed while also indicating that the potential gain naturally diminishes as the two peers become increasingly similar.

\newpage


\section{Experimental Details}
\label{app:experimental details}
\subsection{Implementation and Hyperparameter Details}
\label{appendix:hyperparams}
\begin{wraptable}{r}{0.45\textwidth}
    \centering
    \vspace{-12pt}
    \scriptsize
    \caption{Implementation details. \textcolor{red}{Red}: identical across 
             all applicable benchmarks; \underline{underline}: identical in 
             most benchmarks.}
    \label{tab:hyperparams}
    \setlength{\tabcolsep}{3.5pt}
    \renewcommand{\arraystretch}{1.15}
    \begin{tabularx}{\linewidth}{l | C C C C C}
        \toprule
        & \textbf{O31} & \textbf{OH} & \textbf{DN} & \textbf{VisDA} & \textbf{Cam17} \\
        \midrule
        \multicolumn{6}{l}{\textit{Optimization}} \\
        \hspace{4pt} Epochs
            & 30    & 50    & 50    & 30    & 20    \\
        \hspace{4pt} Batch size
            & 64    & 192   & 192   & 192   & 128   \\
        \hspace{4pt} LR
            & \underline{5e-3} & \underline{1e-2} & \underline{1e-2} 
            & \underline{5e-3} & \underline{5e-3} \\
        \hspace{4pt} Prompt LR
            & 5e-3  & 5e-4  & 5e-4  & 1e-3  & 5e-4  \\
        \midrule
        \multicolumn{6}{l}{\textit{Cut statistic}} \\
        \hspace{4pt} $K$
            & \underline{20} & \underline{50} & \underline{50} 
            & \underline{50} & \underline{20} \\
        \hspace{4pt} EMA $m_{\text{start}}$
            & \textcolor{red}{0.7} & \textcolor{red}{0.7}
            & \textcolor{red}{0.7} & \textcolor{red}{0.7}
            & \textcolor{red}{0.7} \\
        \hspace{4pt} EMA $m_{\text{end}}$
            & \textcolor{red}{0.95} & \textcolor{red}{0.95}
            & \textcolor{red}{0.95} & \textcolor{red}{0.95}
            & \textcolor{red}{0.95} \\
        \midrule
        \multicolumn{6}{l}{\textit{Loss weights}} \\
        \hspace{4pt} $\beta$
            & \underline{3.0} & \underline{6.0} & \underline{6.0} 
            & \underline{6.0} & \underline{5.0}     \\
        \hspace{4pt} $\lambda_{\mathrm{iic}}$
            & \underline{1.3} & \underline{1.3} & \underline{1.3} 
            & \underline{1.3} & \underline{1.0} \\
        \hspace{4pt} $\lambda_{\mathrm{div}}$
            & 0.8   & 0.5   & 0.05  & 0.05  & 0.1   \\
        \midrule
        \multicolumn{6}{l}{\textit{CLIP}} \\
        \hspace{4pt} Prompt init
            & \multicolumn{5}{c}{\texttt{a\_photo\_of\_a}} \\
        \bottomrule
    \end{tabularx}
    \vspace{-8pt}
\end{wraptable}
\paragraph{Hyperparameters}
Table~\ref{tab:hyperparams} summarizes the hyperparameters used across all benchmarks.
A single hyperparameter configuration is applied uniformly to all adaptation tasks 
within each dataset; no per-task search or adjustment is performed.
The EMA schedule for cut statistic estimation is held strictly constant across all 
benchmarks, with $m_{\mathrm{start}}=0.7$ and $m_{\mathrm{end}}=0.95$ following the 
cosine annealing schedule of Equation~\eqref{eq:cosine_momentum}, and the CLIP prompt 
is initialized uniformly as \texttt{a\_photo\_of\_a} throughout. The gating sharpness 
$\beta$, the IIC weight $\lambda_{\mathrm{iic}}$, and the neighborhood size $K$ are 
consistent across most benchmarks, with minor adjustments that follow intuitive 
dataset-scale considerations: smaller datasets (Office-31, Camelyon17) use a smaller 
$K$ due to reduced sample counts, and $\beta$ is slightly lower for Office-31 and 
Camelyon17 given their smaller scale and reduced domain complexity. The diversity 
weight $\lambda_{\mathrm{div}}$ is the only parameter that varies more substantially 
across benchmarks, reflecting differences in class balance and dataset scale rather 
than task-specific tuning.
\paragraph{Baseline Reproduction.}
All baselines are reproduced under a unified experimental setup following the 
protocol of ProDe~\cite{prode_tang_iclr}, using the codebase and pretrained source 
models provided in the ProDe repository.\footnote{\url{https://github.com/tntek/source-free-domain-adaptation}} For methods not included in the ProDe repository, namely ELR~\cite{elr_yi_iclr} and 
AaD~\cite{aad_yang_nips}, we refer to their respective official implementations.\footnote{\url{https://github.com/liyi01827/SFDA_LLN}}~\footnote{\url{https://github.com/Albert0147/AaD_SFDA}} Pretrained source models are taken directly from the ProDe repository without 
modification, ensuring that all methods are evaluated from identical initializations.
All experiments are conducted using PyTorch with a single NVIDIA GPU~\footnote{NVIDIA RTX PRO 6000  Blackwell Server Edition GPU (97\,GB VRAM).}.

\subsection{ProDe Hyperparameter Search on Camelyon17}
\label{appendix:prode_camelyon}

\begin{wraptable}{r}{0.52\textwidth}
    \centering
    \vspace{-12pt}
    \scriptsize
    \caption{Hyperparameter search space for ProDe on Camelyon17.}
    \label{tab:prode_grid}
    \setlength{\tabcolsep}{4pt}
    \renewcommand{\arraystretch}{1.2}
    \begin{tabularx}{\linewidth}{l | X}
        \toprule
        \textbf{Hyperparameter} & \textbf{Search Space} \\
        \midrule
        Epochs $E$
            & $\{1,\ 2,\ 3,\ 4,\ 5,\ 10\}$ \\
        $\lambda_{\mathrm{div}}$
            & $\{0.05,\ 0.1\}$ \\
        $\lambda_{\mathrm{iic}}$
            & $\{0.3,\ 0.5,\ 0.6,\ 0.7,\ 1.0\}$ \\
        LR
            & $\{10^{-3},\ 3{\times}10^{-3},\ 5{\times}10^{-3},\ 10^{-2},$
              $\ 0.05,\ 0.06,\ \ldots,\,\ 0.25,\ 0.30,\ 0.35,\ 0.40\}$ \\
        Prompt LR
            & $\{10^{-5},\ 5{\times}10^{-5},\ 10^{-4},\ 5{\times}10^{-4},$
              $\ 10^{-3},\ 5{\times}10^{-3},\ 0.01,\ 0.02,\ \ldots,\ 0.10\}$ \\
        Prompt init
            & \texttt{a\_photo\_of\_a} (fixed) \\
        \bottomrule
    \end{tabularx}
    \vspace{-8pt}
\end{wraptable}

As noted in Section~\ref{sec:experiments}, ProDe achieves only modest improvement 
on Camelyon17 despite extensive hyperparameter tuning. To ensure a fair comparison, 
we conduct a thorough grid search over all hyperparameters that could plausibly affect 
ProDe's performance in this setting, with the full search space summarized in 
Table~\ref{tab:prode_grid}. The CLIP prompt is initialized as \texttt{a\_photo\_of\_a}, 
identical to the setting used in ProDe's original natural-image experiments. The best 
result obtained across all configurations is $76.27\%$ on the target model, with CLIP 
remaining near its zero-shot performance ($50.00\%$) regardless of the configuration.
We attribute this performance ceiling to a fundamental limitation of ProDe's update 
structure in this setting, rather than suboptimal hyperparameter selection. ProDe 
supervises CLIP exclusively through $\mathcal{L}_{\mathrm{IIC}}$~\citep{iic_ji_iccv}, a self-supervised 
clustering objective that operates without any source-driven signal. When the peer 
model starts from a poor initialization -- as is the case for zero-shot CLIP on 
histopathology, where it underperforms even the source-pretrained model ($50.01\%$ 
vs.\ $64.14\%$) -- $\mathcal{L}_{\mathrm{IIC}}$ alone is insufficient to recover 
meaningful representations, leaving CLIP unable to provide reliable guidance 
throughout training. This one-directional update structure lacks the mechanism to 
bootstrap a weak peer model, making the performance gap robust to hyperparameter 
variation.

\subsection{Full Per-Task Ablation Results}
\label{appendix:ablation_full}
\paragraph{Component Ablation (Table~\ref{tab:ablation_full}).}
Table~\ref{tab:ablation_full} reports the full per-task results for each loss 
configuration. For clarity, we present the explicit training objectives for each case.

\noindent\textbf{$\mathcal{L}_{\mathrm{c}}\!=\!0, \mathcal{L}_{\mathrm{a}}\!=\!0$:}
Cross-model correction and temporal anchoring are both disabled, reducing the 
objective to batch-level regularization only:
\begin{equation}
    \mathcal{L}_{t} = \lambda_{\mathrm{IIC}} \mathcal{L}_{\mathrm{IIC}} - \lambda_{\mathrm{div}} H(\bar{p}_t), \quad
    \mathcal{L}_{v} = \lambda_{\mathrm{IIC}} \mathcal{L}_{\mathrm{IIC}} - \lambda_{\mathrm{div}} H(\bar{p}_v).
\end{equation}

\noindent\textbf{$\mathcal{L}_{\mathrm{c}}\!=\!0, \mathcal{L}_{\mathrm{a}}\!=\!1$:}
Cross-model correction is disabled, but temporal anchoring is retained with uniform 
weighting (gate $\equiv 1$):
\begin{align}
    \mathcal{L}_{t} = \mathbb{E}_{x}\!\left[ D_{\mathrm{KL}}(\tilde{p}_t \| p_t) \right] + \lambda_{\mathrm{IIC}} \mathcal{L}_{\mathrm{IIC}} - \lambda_{\mathrm{div}} H(\bar{p}_t),\\
    \mathcal{L}_{v} = \mathbb{E}_{x}\!\left[ D_{\mathrm{KL}}(\tilde{p}_v \| p_v) \right] + \lambda_{\mathrm{IIC}} \mathcal{L}_{\mathrm{IIC}} - \lambda_{\mathrm{div}} H(\bar{p}_v).
\end{align}

\noindent\textbf{$\mathcal{L}_{\mathrm{c}}\!=\!1, \mathcal{L}_{\mathrm{a}}\!=\!0$:}
Cross-model correction is enabled with uniform weighting, while temporal anchoring 
is disabled:
\begin{align}
    \mathcal{L}_{t} = \mathbb{E}_{x}\!\left[ D_{\mathrm{KL}}(\mathrm{sg}(p_v) \| p_t) \right] + \lambda_{\mathrm{IIC}} \mathcal{L}_{\mathrm{IIC}} - \lambda_{\mathrm{div}} H(\bar{p}_t),\\
    \mathcal{L}_{v} = \mathbb{E}_{x}\!\left[ D_{\mathrm{KL}}(\mathrm{sg}(p_t) \| p_v) \right] + \lambda_{\mathrm{IIC}} \mathcal{L}_{\mathrm{IIC}} - \lambda_{\mathrm{div}} H(\bar{p}_v).
\end{align}

\noindent\textbf{$\mathcal{L}_{\mathrm{c}}\!=\!1, \mathcal{L}_{\mathrm{a}}\!=\!1$:}
Both cross-model correction and temporal anchoring are enabled with uniform weighting:
\begin{align}
    \mathcal{L}_{t} &= \mathbb{E}_{x}\!\left[ D_{\mathrm{KL}}(\mathrm{sg}(p_v) \| p_t) + D_{\mathrm{KL}}(\tilde{p}_t \| p_t) \right] + \lambda_{\mathrm{IIC}} \mathcal{L}_{\mathrm{IIC}} - \lambda_{\mathrm{div}} H(\bar{p}_t), \\
    \mathcal{L}_{v} &= \mathbb{E}_{x}\!\left[ D_{\mathrm{KL}}(\mathrm{sg}(p_t) \| p_v) + D_{\mathrm{KL}}(\tilde{p}_v \| p_v) \right] + \lambda_{\mathrm{IIC}} \mathcal{L}_{\mathrm{IIC}} - \lambda_{\mathrm{div}} H(\bar{p}_v).
\end{align}

\noindent\textbf{\proposed (Ours):}
Both terms are enabled with adaptive reliability gating $T_{v \to t}$ and 
$T_{t \to v}$, replacing the uniform weights above:
\begin{align}
    \mathcal{L}_{t} &= \mathbb{E}_{x}\!\left[ T_{v \to t} \cdot D_{\mathrm{KL}}(\mathrm{sg}(p_v) \| p_t) + T_{t \to v} \cdot D_{\mathrm{KL}}(\tilde{p}_t \| p_t) \right] + \lambda_{\mathrm{IIC}} \mathcal{L}_{\mathrm{IIC}} - \lambda_{\mathrm{div}} H(\bar{p}_t), \\
    \mathcal{L}_{v} &= \mathbb{E}_{x}\!\left[ T_{t \to v} \cdot D_{\mathrm{KL}}(\mathrm{sg}(p_t) \| p_v) + T_{v \to t} \cdot D_{\mathrm{KL}}(\tilde{p}_v \| p_v) \right] + \lambda_{\mathrm{IIC}} \mathcal{L}_{\mathrm{IIC}} - \lambda_{\mathrm{div}} H(\bar{p}_v).
\end{align}
\proposed consistently outperforms all ablated variants across all three benchmarks, 
confirming the complementary roles of reliability-aware cross-model correction, temporal anchoring.

\begin{table*}[h]
    \centering
    \scriptsize
    \caption{Full per-task component ablation results (acc, \%). 
    $\mathcal{L}_{\mathrm{c}}$: cross-supervision KL, 
    $\mathcal{L}_{\mathrm{a}}$: temporal anchoring.}
    \label{tab:ablation_full}
    \renewcommand{\arraystretch}{1.2}

    \textbf{(a) Office-31}\\[3pt]
    \begin{tabularx}{\textwidth}{l | *{6}{C} | C}
        \toprule
        Method & \tiny{A$\to$D} & \tiny{A$\to$W} & \tiny{D$\to$A} & \tiny{D$\to$W} & \tiny{W$\to$A} & \tiny{W$\to$D} & Avg. \\
        \midrule
        $\mathcal{L}_{\mathrm{c}}\!=\!0, \mathcal{L}_{\mathrm{a}}\!=\!0$ & 97.2 & 91.7 & 81.2 & 97.4 & 80.6 & 99.0 & 91.2 \\
        $\mathcal{L}_{\mathrm{c}}\!=\!0, \mathcal{L}_{\mathrm{a}}\!=\!1$ & 97.2 & 92.2 & 80.9 & 97.5 & 79.6 & 99.8 & 91.2 \\
        $\mathcal{L}_{\mathrm{c}}\!=\!1, \mathcal{L}_{\mathrm{a}}\!=\!0$ & 97.0 & 94.2 & 83.7 & 96.7 & \textbf{\textcolor{red}{84.2}} & 98.8 & 92.4 \\
        $\mathcal{L}_{\mathrm{c}}\!=\!1, \mathcal{L}_{\mathrm{a}}\!=\!1$ & 96.2 & 94.0 & 83.6 & 96.5 & 83.5 & 99.0 & 92.1 \\
        \midrule
        \rowcolor{gray!35}
        \proposed & \textbf{\textcolor{red}{98.6}} & \textbf{\textcolor{red}{95.2}} & \textbf{\textcolor{red}{83.9}} & \textbf{\textcolor{red}{98.2}} & 83.3 & \textbf{\textcolor{red}{100.0}} & \textbf{\textcolor{red}{93.2}} \\
        \bottomrule
    \end{tabularx}
    \vspace{8pt}

    \textbf{(b) Office-Home}\\[3pt]
    \begin{tabularx}{\textwidth}{l | *{12}{C} | C}
        \toprule
        Method & \tiny{Ar$\to$Cl} & \tiny{Ar$\to$Pr} & \tiny{Ar$\to$Rw} & \tiny{Cl$\to$Ar} & \tiny{Cl$\to$Pr} & \tiny{Cl$\to$Rw} & \tiny{Pr$\to$Ar} & \tiny{Pr$\to$Cl} & \tiny{Pr$\to$Rw} & \tiny{Rw$\to$Ar} & \tiny{Rw$\to$Cl} & \tiny{Rw$\to$Pr} & Avg. \\
        \midrule
        $\mathcal{L}_{\mathrm{c}}\!=\!0, \mathcal{L}_{\mathrm{a}}\!=\!0$ & 69.6 & 84.5 & 90.3 & 80.8 & 85.6 & 89.7 & 78.2 & 68.9 & 89.5 & 83.3 & 70.4 & 90.6 & 81.8 \\
        $\mathcal{L}_{\mathrm{c}}\!=\!0, \mathcal{L}_{\mathrm{a}}\!=\!1$ & 71.2 & 88.2 & 90.7 & 81.8 & 90.3 & 90.0 & 79.6 & 71.1 & 89.5 & 83.4 & 71.8 & 92.2 & 83.3 \\
        $\mathcal{L}_{\mathrm{c}}\!=\!1, \mathcal{L}_{\mathrm{a}}\!=\!0$ & 74.7 & 92.7 & 91.7 & 85.2 & \textbf{\textcolor{red}{92.7}} & 91.4 & 84.6 & 75.8 & 91.4 & 84.5 & 74.4 & 92.7 & 86.0 \\
        $\mathcal{L}_{\mathrm{c}}\!=\!1, \mathcal{L}_{\mathrm{a}}\!=\!1$ & 76.2 & 92.5 & 91.5 & 85.7 & 92.7 & 91.3 & \textbf{\textcolor{red}{85.0}} & 76.1 & 91.3 & 84.8 & 75.1 & 92.5 & 86.2 \\
        \midrule
        \rowcolor{gray!35}
        \proposed & \textbf{\textcolor{red}{76.9}} & \textbf{\textcolor{red}{93.0}} & \textbf{\textcolor{red}{91.9}} & \textbf{\textcolor{red}{85.9}} & 92.5 & \textbf{\textcolor{red}{91.6}} & \textbf{\textcolor{red}{85.0}} & \textbf{\textcolor{red}{76.7}} & \textbf{\textcolor{red}{92.0}} & \textbf{\textcolor{red}{85.7}} & \textbf{\textcolor{red}{76.1}} & \textbf{\textcolor{red}{93.2}} & \textbf{\textcolor{red}{86.7}} \\
        \bottomrule
    \end{tabularx}
    \vspace{8pt}

    \textbf{(c) DomainNet-126}\\[3pt]
    \begin{tabularx}{\textwidth}{l | *{12}{C} | C}
        \toprule
        Method & \tiny{C$\to$P} & \tiny{C$\to$R} & \tiny{C$\to$S} & \tiny{P$\to$C} & \tiny{P$\to$R} & \tiny{P$\to$S} & \tiny{R$\to$C} & \tiny{R$\to$P} & \tiny{R$\to$S} & \tiny{S$\to$C} & \tiny{S$\to$P} & \tiny{S$\to$R} & Avg. \\
        \midrule
        $\mathcal{L}_{\mathrm{c}}\!=\!0, \mathcal{L}_{\mathrm{a}}\!=\!0$ & 63.1 & 78.2 & 62.5 & 71.2 & 86.5 & 69.0 & 71.1 & 72.0 & 64.9 & 73.2 & 72.0 & 81.8 & 72.1 \\
        $\mathcal{L}_{\mathrm{c}}\!=\!0, \mathcal{L}_{\mathrm{a}}\!=\!1$ & 71.3 & 85.9 & 70.9 & 78.9 & 89.0 & 73.8 & 81.1 & 78.8 & 73.8 & 79.6 & 74.9 & 86.1 & 78.7 \\
        $\mathcal{L}_{\mathrm{c}}\!=\!1, \mathcal{L}_{\mathrm{a}}\!=\!0$ & 80.6 & 91.7 & 78.9 & 83.3 & 91.5 & 78.0 & 83.9 & 80.8 & 78.2 & 83.9 & 80.9 & 91.8 & 83.6 \\
        $\mathcal{L}_{\mathrm{c}}\!=\!1, \mathcal{L}_{\mathrm{a}}\!=\!1$ & 81.7 & 91.9 & 79.3 & 84.2 & 92.0 & 79.1 & 84.2 & 81.6 & 79.1 & 84.3 & 81.9 & 92.1 & 84.3 \\
        \midrule
        \rowcolor{gray!35}
        \proposed & \textbf{\textcolor{red}{82.8}} & \textbf{\textcolor{red}{92.3}} & \textbf{\textcolor{red}{81.8}} & \textbf{\textcolor{red}{86.2}} & \textbf{\textcolor{red}{92.2}} & \textbf{\textcolor{red}{81.0}} & \textbf{\textcolor{red}{86.9}} & \textbf{\textcolor{red}{82.9}} & \textbf{\textcolor{red}{81.7}} & \textbf{\textcolor{red}{86.6}} & \textbf{\textcolor{red}{83.0}} & \textbf{\textcolor{red}{92.3}} & \textbf{\textcolor{red}{85.8}} \\
        \bottomrule
    \end{tabularx}
\end{table*}

\paragraph{Design Variant Comparison (Table~\ref{tab:variant_full}).}
Table~\ref{tab:variant_full} compares three design variants against \proposed and an 
oracle upper bound. \textbf{Frozen $\theta_v$} holds the peer model fixed throughout 
adaptation (i.e., $\mathcal{L}_v \equiv 0$), so that $\theta_v$ remains at its 
zero-shot initialization. The consistent performance drop confirms that joint peer 
adaptation is essential: a static peer provides increasingly stale supervision as the 
target model evolves. \textbf{Entropy Gating} replaces the cut-based reliability 
weights with entropy-derived gates:
\begin{equation}
    T^{\mathrm{ent}}_{i}(x) = \exp\!\left(-H(p_i(x))\right), \quad i \in \{t, v\},
\end{equation}
where $H(p_i(x)) = -\sum_c p_i^c(x) \log p_i^c(x)$. While entropy captures 
per-sample confidence from the marginal output distribution, it does not incorporate 
local neighborhood structure. The consistent underperformance relative to \proposed 
supports the cut statistic as a more reliable reliability proxy. 
\textbf{Oracle Partition $\mathcal{S}$} provides an upper bound by assuming access to 
ground-truth correctness labels, applying cross-model supervision exclusively via 
cross-entropy on the subset where the supervising model is known to be correct:
\begin{align}
    \mathcal{L}^{\mathrm{oracle}}_{t} = 
    \mathbb{E}_{x \in \mathcal{S}_{\bar{c}c}} \!\left[ 
        -\!\sum_c \mathbbm{1}[c = \hat{y}_v(x)] \log p_t^c(x)
    \right], \\
    \mathcal{L}^{\mathrm{oracle}}_{v} = 
    \mathbb{E}_{x \in \mathcal{S}_{c\bar{c}}} \!\left[ 
        -\!\sum_c \mathbbm{1}[c = \hat{y}_t(x)] \log p_v^c(x)
    \right].
\end{align}
The gap between \proposed and Oracle $\mathcal{S}$ (e.g., $86.7\%$ vs.\ $93.3\%$ on 
Office-Home) establishes a concrete upper bound for reliability-gated mutual 
correction, and suggests that more precise approximations of the oracle partition 
remain a promising direction for future work.

\begin{table*}[h]
    \centering
    \scriptsize
    \caption{Full per-task variant comparison results (acc, \%).}
    \label{tab:variant_full}
    \renewcommand{\arraystretch}{1.2}

    \textbf{(a) Office-31}\\[3pt]
    \begin{tabularx}{\textwidth}{l | *{6}{C} | C}
        \toprule
        Variant & \tiny{A$\to$D} & \tiny{A$\to$W} & \tiny{D$\to$A} & \tiny{D$\to$W} & \tiny{W$\to$A} & \tiny{W$\to$D} & Avg. \\
        \midrule
        Frozen $\theta_v$              & 90.8 & 90.8 & 81.5 & 94.3 & 81.3 & 95.4  & 89.0 \\
        Entropy Gating                 & 97.8 & 94.2 & \textbf{\textcolor{red}{83.9}} & 97.5 & \textbf{\textcolor{red}{83.6}} & 98.4 & 92.5 \\
        \midrule
        \rowcolor{gray!35}
        \proposed                      & \textbf{\textcolor{red}{98.6}} & \textbf{\textcolor{red}{95.2}} & \textbf{\textcolor{red}{83.9}} & \textbf{\textcolor{red}{98.2}} & 83.3 & \textbf{\textcolor{red}{100.0}} & \textbf{\textcolor{red}{93.2}} \\
        \midrule
        Oracle partition $\mathcal{S}$ & 99.4 & 99.9 & 88.8 & 99.3 & 88.8 & 100.0 & 96.0 \\
        \bottomrule
    \end{tabularx}
    \vspace{8pt}

    \textbf{(b) Office-Home}\\[3pt]
    \begin{tabularx}{\textwidth}{l | *{12}{C} | C}
        \toprule
        Variant & \tiny{Ar$\to$Cl} & \tiny{Ar$\to$Pr} & \tiny{Ar$\to$Rw} & \tiny{Cl$\to$Ar} & \tiny{Cl$\to$Pr} & \tiny{Cl$\to$Rw} & \tiny{Pr$\to$Ar} & \tiny{Pr$\to$Cl} & \tiny{Pr$\to$Rw} & \tiny{Rw$\to$Ar} & \tiny{Rw$\to$Cl} & \tiny{Rw$\to$Pr} & Avg. \\
        \midrule
        Frozen $\theta_v$              & 73.5 & 90.4 & 90.9 & 84.2 & 90.7 & 90.6 & 83.6 & 73.4 & 91.1 & 84.1 & 74.4 & 90.8 & 84.8 \\
        Entropy Gating                 & 75.7 & 92.7 & \textbf{\textcolor{red}{92.0}} & 85.5 & 92.3 & \textbf{\textcolor{red}{91.7}} & \textbf{\textcolor{red}{85.0}} & 76.1 & 91.8 & 85.1 & 75.7 & 92.3 & 86.3 \\
        \midrule
        \rowcolor{gray!35}
        \proposed                      & \textbf{\textcolor{red}{76.9}} & \textbf{\textcolor{red}{93.0}} & \textbf{\textcolor{red}{92.0}} & \textbf{\textcolor{red}{85.9}} & \textbf{\textcolor{red}{92.5}} & 91.6 & \textbf{\textcolor{red}{85.0}} & \textbf{\textcolor{red}{76.7}} & \textbf{\textcolor{red}{92.0}} & \textbf{\textcolor{red}{85.7}} & \textbf{\textcolor{red}{76.1}} & \textbf{\textcolor{red}{93.2}} & \textbf{\textcolor{red}{86.7}} \\
        \midrule
        Oracle partition $\mathcal{S}$ & 88.0 & 97.0 & 95.8 & 91.9 & 97.1 & 95.7 & 92.0 & 87.8 & 95.8 & 93.0 & 88.2 & 97.3 & 93.3 \\
        \bottomrule
    \end{tabularx}
    \vspace{8pt}

    \textbf{(c) DomainNet-126}\\[3pt]
    \begin{tabularx}{\textwidth}{l | *{12}{C} | C}
        \toprule
        Variant & \tiny{C$\to$P} & \tiny{C$\to$R} & \tiny{C$\to$S} & \tiny{P$\to$C} & \tiny{P$\to$R} & \tiny{P$\to$S} & \tiny{R$\to$C} & \tiny{R$\to$P} & \tiny{R$\to$S} & \tiny{S$\to$C} & \tiny{S$\to$P} & \tiny{S$\to$R} & Avg. \\
        \midrule
        Frozen $\theta_v$              & 81.8 & 91.7 & 80.0 & 84.7 & 91.7 & 80.0 & 84.6 & 81.9 & 79.5 & 85.0 & 81.8 & 91.8 & 84.6 \\
        Entropy Gating                 & 82.4 & 92.2 & 80.2 & 85.1 & \textbf{\textcolor{red}{92.2}} & 80.2 & 85.2 & 82.5 & 80.2 & 84.9 & 82.4 & 92.2 & 85.0 \\
        \midrule
        \rowcolor{gray!35}
        \proposed                      & \textbf{\textcolor{red}{82.8}} & \textbf{\textcolor{red}{92.3}} & \textbf{\textcolor{red}{81.8}} & \textbf{\textcolor{red}{86.2}} & \textbf{\textcolor{red}{92.2}} & \textbf{\textcolor{red}{81.0}} & \textbf{\textcolor{red}{86.9}} & \textbf{\textcolor{red}{82.9}} & \textbf{\textcolor{red}{81.7}} & \textbf{\textcolor{red}{86.6}} & \textbf{\textcolor{red}{83.0}} & \textbf{\textcolor{red}{92.3}} & \textbf{\textcolor{red}{85.8}} \\
        \midrule
        Oracle partition $\mathcal{S}$ & 91.2 & 96.1 & 90.2 & 93.4 & 96.0 & 90.2 & 93.3 & 91.2 & 90.1 & 93.5 & 91.2 & 96.1 & 92.7 \\
        \bottomrule
    \end{tabularx}
\end{table*}

\paragraph{CLIP Adaptation.}
Table~\ref{tab:clip_adaptation_full} compares zero-shot CLIP against the peer model 
adapted within \proposed across all tasks. Despite zero-shot CLIP's strong 
vision-language pretraining, prompt learning within \proposed consistently and 
substantially improves peer accuracy across all benchmarks: $+13.4\%$ on Office-31 
($79.4\% \to 92.8\%$), $+7.7\%$ on Office-Home ($79.0\% \to 86.7\%$), and $+3.8\%$ 
on DomainNet-126 ($80.2\% \to 84.0\%$). These gains confirm that target-domain 
specialization of the peer model is essential: an adapted peer provides 
higher-fidelity supervision signals throughout mutual correction, and the improvement 
in peer accuracy directly translates to stronger correction of the target model.

\begin{table*}[h]
    \centering
    \scriptsize
    \caption{Full per-task CLIP adaptation results (acc, \%). Zero-shot: CLIP without 
    adaptation. Adaptation: CLIP adapted as peer model in \proposed.}
    \label{tab:clip_adaptation_full}
    \renewcommand{\arraystretch}{1.2}

    \textbf{(a) Office-31}\\[3pt]
    \begin{tabularx}{\textwidth}{l | *{6}{C} | C}
        \toprule
        & \tiny{A$\to$D} & \tiny{A$\to$W} & \tiny{D$\to$A} & \tiny{D$\to$W} & \tiny{W$\to$A} & \tiny{W$\to$D} & Avg. \\
        \midrule
        Zero-shot  & 81.1 & 80.0 & 77.0 & 80.0 & 77.0 & 81.1 & 79.4 \\
        \rowcolor{gray!35}
        Adaptation & \textbf{\textcolor{red}{98.6}} & \textbf{\textcolor{red}{94.8}} & \textbf{\textcolor{red}{83.5}} & \textbf{\textcolor{red}{98.0}} & \textbf{\textcolor{red}{82.0}} & \textbf{\textcolor{red}{99.8}} & \textbf{\textcolor{red}{92.8}} \\
        \bottomrule
    \end{tabularx}
    \vspace{8pt}

    \textbf{(b) Office-Home}\\[3pt]
    \begin{tabularx}{\textwidth}{l | *{12}{C} | C}
        \toprule
        & \tiny{Ar$\to$Cl} & \tiny{Ar$\to$Pr} & \tiny{Ar$\to$Rw} & \tiny{Cl$\to$Ar} & \tiny{Cl$\to$Pr} & \tiny{Cl$\to$Rw} & \tiny{Pr$\to$Ar} & \tiny{Pr$\to$Cl} & \tiny{Pr$\to$Rw} & \tiny{Rw$\to$Ar} & \tiny{Rw$\to$Cl} & \tiny{Rw$\to$Pr} & Avg. \\
        \midrule
        Zero-shot  & 64.6 & 86.4 & 87.2 & 77.9 & 86.4 & 87.2 & 77.9 & 64.6 & 87.2 & 77.9 & 64.6 & 86.4 & 79.0 \\
        \rowcolor{gray!35}
        Adaptation & \textbf{\textcolor{red}{76.9}} & \textbf{\textcolor{red}{93.0}} & \textbf{\textcolor{red}{91.9}} & \textbf{\textcolor{red}{85.9}} & \textbf{\textcolor{red}{92.5}} & \textbf{\textcolor{red}{91.6}} & \textbf{\textcolor{red}{85.0}} & \textbf{\textcolor{red}{76.7}} & \textbf{\textcolor{red}{92.0}} & \textbf{\textcolor{red}{85.7}} & \textbf{\textcolor{red}{76.1}} & \textbf{\textcolor{red}{93.2}} & \textbf{\textcolor{red}{86.7}} \\
        \bottomrule
    \end{tabularx}
    \vspace{8pt}

    \textbf{(c) DomainNet-126}\\[3pt]
    \begin{tabularx}{\textwidth}{l | *{12}{C} | C}
        \toprule
        & \tiny{C$\to$P} & \tiny{C$\to$R} & \tiny{C$\to$S} & \tiny{P$\to$C} & \tiny{P$\to$R} & \tiny{P$\to$S} & \tiny{R$\to$C} & \tiny{R$\to$P} & \tiny{R$\to$S} & \tiny{S$\to$C} & \tiny{S$\to$P} & \tiny{S$\to$R} & Avg. \\
        \midrule
        Zero-shot  & 77.5 & 89.8 & 74.5 & 78.9 & 89.8 & 74.5 & 78.9 & 77.5 & 74.5 & 78.9 & 77.5 & 89.8 & 80.2 \\
        \rowcolor{gray!35}
        Adaptation & \textbf{\textcolor{red}{81.1}} & \textbf{\textcolor{red}{91.8}} & \textbf{\textcolor{red}{79.1}} & \textbf{\textcolor{red}{84.3}} & \textbf{\textcolor{red}{91.6}} & \textbf{\textcolor{red}{78.5}} & \textbf{\textcolor{red}{84.6}} & \textbf{\textcolor{red}{80.9}} & \textbf{\textcolor{red}{79.2}} & \textbf{\textcolor{red}{84.2}} & \textbf{\textcolor{red}{81.5}} & \textbf{\textcolor{red}{91.8}} & \textbf{\textcolor{red}{84.0}} \\
        \bottomrule
    \end{tabularx}
\end{table*}
\clearpage


\section{Limitations and Future Directions}
\label{appendix:limitations}

\proposed establishes a principled foundation for reliability-gated mutual correction, while several directions remain open for further exploration. The framework benefits from complementary prediction patterns between the two peers, while increasingly similar predictions provide fewer opportunities for cross-model correction. Samples on which both models are incorrect also remain comparatively challenging, as neither peer provides an immediately corrective signal, although some of these samples are gradually improved through changes in the surrounding neighborhood structure. These observations suggest several promising extensions. In particular, samples in low-margin regions could benefit from complementary reliability cues, such as prediction entropy or temporal consistency, to further sharpen directional discrimination. Incorporating additional pseudo-label-based strategies~\citep{dividemix_li_iclr, flexmatch_zhang_nips} could also provide richer supervisory signals during adaptation and potentially move performance closer to the oracle correction bound identified in Table~\ref{tab:analysis}(b). We leave these directions as future work.

\section{Broader Impacts}
\label{app:broader_impacts}
\proposed contributes to the broader goal of deploying machine learning models reliably across diverse real-world environments without requiring labeled target data or access to sensitive source data. This has direct positive implications in privacy-sensitive domains such as healthcare and finance, where sharing training data is often legally or ethically restricted. By enabling effective adaptation under such constraints, \proposed lowers the barrier to deploying high-quality models in underserved or data-scarce settings. Beyond privacy, the source-free constraint also reduces the computational and logistical costs associated with retaining and transferring large-scale source datasets, making model deployment more accessible to practitioners with limited infrastructure. Furthermore, the mutual correction framework introduced in \proposed is model-agnostic and does not assume a specific type of peer model, broadening its applicability to heterogeneous deployment scenarios where diverse pretrained models are increasingly available. Finally, by establishing a theoretically grounded reliability criterion for cross-model supervision, \proposed contributes to the broader agenda of trustworthy machine learning, providing principled mechanisms to mitigate confirmation bias and error propagation in self-supervised adaptation pipelines. We do not foresee significant negative societal impacts from this work. However, as with any adaptation method that operates without human supervision, there remains a risk that errors shared across both peer models go undetected, potentially leading to incorrect predictions in high-stakes applications such as medical diagnosis. Practitioners deploying \proposed in such settings should complement adaptation with appropriate human oversight and validation protocols.

\section{Asset Licenses and Terms of Use}
\label{app:asset_licenses}

\paragraph{Datasets.}
Office-31~\citep{office31_saenko_eccv} does not specify an explicit license and is used in this work solely for academic research purposes, with proper attribution to the original authors. Office-Home~\citep{officehome_venkateswara_cvpr} is distributed under a fair-use notice for non-commercial research and educational purposes. DomainNet-126~\citep{domainnet_saito_iccv} contains copyrighted material distributed under a fair-use notice for non-commercial research and educational purposes; uses beyond this scope may require permission from the respective copyright holders. VisDA~\citep{visda_peng_arxiv} is used for non-commercial research and educational purposes in accordance with its usage terms. Camelyon17 from the WILDS benchmark~\citep{wilds_sagawa_arxiv} is distributed under the Creative Commons CC0 1.0 Universal license (public domain). EyePACS~\citep{cuadros_eyepacs_journal} is used for academic research with attribution to EyePACS, as requested by the dataset provider for research and publication use. APTOS 2019 is used under the competition rules permitting non-commercial use of the competition data for academic research and education.

\paragraph{Models.} CLIP~\citep{clip_radford_icml} is released by OpenAI under the MIT license. The ResNet-50 and ResNet-101 implementations used in our experiments are provided via torchvision, which is released under the BSD 3-Clause license. CoOp~\citep{coop_zhou_ijcv} is released under the MIT license.

\end{document}